\documentclass[conference]{IEEEtran}
\usepackage{graphicx,cite,calc}
\usepackage[usenames,dvipsnames]{color}
\usepackage{subfigure}
\usepackage{amsthm,amsfonts,amssymb}
\usepackage{hyperref}
\usepackage{subeqn}
\usepackage{tikz}
\usetikzlibrary{patterns}
\usetikzlibrary{fit,calc,positioning,decorations.pathreplacing,matrix}

\newtheorem{theorem}{Theorem}

\newtheorem{definition}{Definition}

\newtheorem{algorithm}{Algorithm}

\newtheorem{lemma}[theorem]{Lemma}

\newcommand{\beq}{\begin{equation}}
\newcommand{\eeq}{\end{equation}}
\newcommand{\beqa}{\begin{eqnarray}}
\newcommand{\eeqa}{\end{eqnarray}}
\newcommand{\bit}{\begin{itemize}}
\newcommand{\eit}{\end{itemize}}
\newcommand{\ben}{\begin{enumerate}}
\newcommand{\een}{\end{enumerate}}
\newcommand{\mc}{\mathcal}
\newcommand{\mb}{\mathbb}
\newcommand{\bed}{\begin{displaymath}}
\newcommand{\eed}{\end{displaymath}}
\newtheorem{thm}{Theorem}

\newtheorem{lem}[thm]{Lemma}
\newtheorem{rem}[thm]{Remark}

\usepackage{algorithm,algorithmic}

\tikzstyle{line}=[draw]
\begin{document}

\title{A Non-Binary Associative Memory with Exponential Pattern Retrieval Capacity and Iterative Learning }

\author{\IEEEauthorblockN{Amir Hesam Salavati$^\dagger$, K. Raj Kumar$^\ddagger$, and Amin Shokrollahi$^\dagger$}
\IEEEauthorblockA{$\dagger$:Laboratoire d'algorithmique (ALGO)\\
Ecole Polytechnique F\'{e}d\'{e}rale de Lausanne (EPFL), 1015
Lausanne, Switzerland\\ E-mail:
\{hesam.salavati,amin.shokrollahi\}@epfl.ch}
\IEEEauthorblockA{$\ddagger$:Qualcomm Research India\\
Bangalore - 560066, India\\ E-mail:
kumarraj@qti.qualcomm.com}}

\maketitle

\begin{abstract}
We consider the problem of neural association for a network of
non-binary neurons. Here, the task is to first memorize a set of patterns using a network of neurons whose states
assume values from a finite number of integer levels. Later, the same network should be able to recall previously memorized patterns from their noisy versions. Prior work in this area
consider storing a finite number of
{\em purely random} patterns, and have shown that the pattern
retrieval capacities (maximum number of patterns that can be
memorized) scale only linearly with the number of neurons in the
network.  

In our formulation of the problem, we concentrate on exploiting redundancy and internal structure of the patterns in order to improve the pattern retrieval capacity. Our first result shows that if the given patterns have a suitable linear-algebraic structure, i.e. comprise a sub-space of the set of all possible patterns, then the pattern retrieval capacity is in fact exponential in terms of the number of neurons. The second result extends the previous finding to cases where the patterns have weak minor components, i.e.
the smallest eigenvalues of the correlation matrix tend toward zero. We will use these minor components (or the basis vectors of the pattern null space) to both increase the pattern
retrieval capacity and error correction capabilities. 

An iterative algorithm is proposed for the learning phase, and two simple neural update algorithms are presented for the recall phase. Using analytical results and simulations,
we show that the proposed methods can tolerate a fair amount of errors in the input while being able to memorize an exponentially large number of patterns. 
\end{abstract}

\section{Introduction}\label{section_introduction}
Neural associative memory is a particular class of neural networks capable of memorizing (learning) a set of patterns and recalling them later in presence of noise, i.e. retrieve the
correct memorized pattern from a given noisy version. Starting from the seminal work of Hopfield in 1982 \cite{hopfield}, various artificial neural networks have been designed to mimic the
task of the neuronal associative memory (see for instance \cite{venkatesh}, \cite{Jankowski}, \cite{Muezzinoglu1}, \cite{SKGS}, \cite{gripon_sparse}). 

In essence, the neural associative memory problem is very similar to the one faced in communication systems where the goal is to reliably and efficiently retrieve a set of patterns
(so called codewords) form noisy versions. More interestingly, the
techniques used to implement an artificial neural associative memory looks very similar to some of the methods used in graph-based modern codes to decode information. This makes the
pattern retrieval phase in neural associative memories very similar to iterative decoding techniques in modern coding theory. 

However, despite the similarity in the task and techniques employed in both problems, there is a huge gap in terms of efficiency. Using binary codewords of length $n$, one can construct codes that are
capable of reliably transmitting $2^{rn}$ codewords over a noisy channel, where $ 0 < r < 1$ is the code rate \cite{urbanke}. The optimal $r$ (i.e. the
largest possible value that permits the almost sure recovery of transmitted codewords from the corrupted received versions) depends on the noise characteristics of the channel and is known as the Shannon capacity \cite{shannon}. In fact, the Shannon capacity is achievable in certain cases, for example by LDPC codes over AWGN channels. 

In current neural associative memories, however, with a network of size $n$ one can only memorize $O(n)$ binary patterns of length $n$ \cite{mceliece}, \cite{venkatesh}. To be fair, it must
be mentioned that these networks are designed such that they are able to memorize any possible set of \emph{randomly} chosen patterns (with size $O(n)$ of course) (e.g., \cite{hopfield},
\cite{venkatesh}, \cite{Jankowski}, \cite{Muezzinoglu1}). Therefore, although humans cannot memorize random patterns, these methods provide artificial neural associative memories with a
pleasant sense of generality.

However, this generality severely restricts the efficiency of the network since even if the input patterns have some internal redundancy or structure, current neural associative memories
could not exploit this redundancy in order to increase the number of memorizable patterns or improve error correction during the recall phase. In fact, concentrating on redundancies within
patterns is a fairly new viewpoint. This point of view is in harmony to coding techniques where one designs codewords with certain degree of redundancy and then use this redundancy to
correct corrupted signals at the receiver's side. 

In this paper, we focus on bridging the performance gap between the coding techniques and neural associative memories. Our proposed neural network exploits the inherent structure of the input
patterns in order to increase the pattern retrieval capacity from $O(n)$ to $O(a^n)$ with $a >1$. More specifically, the proposed neural network is capable of
learning and reliably recalling given patterns when they come from a subspace with dimension $k < n$ of all possible $n$-dimensional patterns. Note that although the proposed model does not
have the versatility of traditional associative memories to handle any set of inputs, such as the Hopfield network \cite{hopfield}, it enables us to boost the capacity by a great extent in
cases where there is some input redundancy. In contrast, traditional associative memories will still have linear pattern retrieval capacity even if the patterns good linear algebraic structures. 

In \cite{KSS}, we presented some preliminary results in which two efficient recall algorithms were proposed for the case where the neural graph had the structure of an expander
\cite{expander_ref}. Here, we extend the previous results to general sparse neural graphs as well as proposing a simple learning algorithm to capture the internal structure of the patterns (which will be used later in the recall phase).

The remainder of this paper is organized as follows: In Section \ref{section_formulation}, we will discuss the neural model used in this paper and formally define the associative memory problem. We explain
the proposed learning algorithm in Section \ref{section_learning}. Sections \ref{section_recall} and \ref{section_error_analysis} are respectively dedicated to the recall algorithm and analytically investigating its performance in
retrieving corrupted patterns. In Section \ref{section_capacity} we address the pattern retrieval capacity and show that it is exponential in $n$. Simulation results are discussed in Section
\ref{section_simulations}. Section \ref{section_conclusion} concludes the paper and discusses future research topics. Finally, the Appendices contain some extra remarks as well as the proofs for certain lemmas and theorems.

\section{Problem Formulation and the Neural Model}\label{section_formulation}
\subsection{The Model}
In the proposed model, we work with neurons whose states are integers from a finite set of non-negative values $\mc{Q} = \{0,1,\dots,Q-1\}$. A natural way of interpreting this model is to
think of the integer states as the short-term firing rate of neurons (possibly quantized). In other words, the state of a neuron in this model indicates the number of spikes fired by the neuron in a fixed short
time interval. 

Like in other neural networks, neurons can only perform simple operations. We consider neurons that can do \emph{linear summation} over the input and possibly apply a \emph{non-linear function}
(such as thresholding) to produce the output. More specifically, neuron $x$ updates its state based on the states of
its neighbors $\{s_i\}_{i=1}^{n}$ as follows:
\begin{enumerate}
\item It computes the weighted sum
$ h = \sum_{i=1}^{n} w_i s_i,$ 
where $w_i$ denotes the weight of the input link from the $i^{th}$ neighbor.
\item It updates its state as $x = f(h),$
where $f: \mb{R} \rightarrow \mc{Q}$ is a possibly non-linear function
from the field of real numbers $\mb{R}$ to $\mc{Q}$.
\end{enumerate}
We will refer to these two as "neural operations" in the sequel.

\subsection{The Problem}
The neural associative memory problem consists of two parts: learning and pattern retrieval. 
\subsubsection{The learning phase}
We assume to be given $C$ vectors of length $n$ with integer-valued entries belonging to $\mc{Q}$. Furthermore, we assume these patterns belong to a subspace of $\mc{Q}^n$ with dimension $k
\leq n$. Let $\mc{X}_{C \times n}$ be the matrix that contains the set of patterns in its rows. Note that if $k = n$, then we are back to the original associative memory problem. However, our focus will beon the case where $k <n$, which will be shown to yield much larger pattern retrieval capacities. Let us denote the model specification by a triplet $(\mc{Q},n,k)$.

The learning phase then comprises a set of steps to determine the connectivity of the neural graph (i.e. finding a set of weights) as a function of the training patterns in $\mc{X}$ such that these patterns are stable states of the recall process. More specifically, in the learning phase we would like to memorize the patterns in $\mc{X}$ by finding a set of non-zero vectors $w_1,\dots,w_m \in \mb{R}^n$ that are orthogonal to the set of given patterns. Remark here that such vectors exist (for instance the basis of the null-space). 

Our interest is to come up with a neural scheme to determine these vectors. Therefore, the inherent structure of the patterns are captured in the obtained null-space vectors, denoted by the matrix $W \in \mb{R}^{m
\times n}$, whose $i^{\mbox{th}}$ row is $w_i$. This matrix can be interpreted as the adjacency matrix of a bipartite graph which represents our neural network. The graph is comprised on pattern and constraint neurons (nodes). Pattern neurons, as they name suggest, correspond to the states of the patterns we would like to learn or recall. The constrain neurons, on the other hand, should verify if the current pattern belongs to the database $\mc{X}$. If not, they should send proper feedback messages to the pattern neurons in order to help them converge to the correct pattern in the dataset. The overall network model is shown in Figure \ref{single_level_net}. 
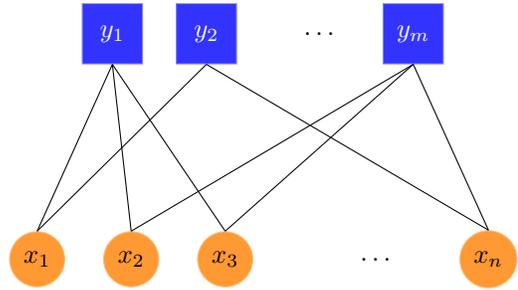
\begin{figure}
\centering
\begin{tikzpicture}
\node at (0,0)[yshift=1.5cm,xshift=-2.5cm,rectangle,,draw=blue!50,fill=blue!80,minimum size=8mm] (c1) {\textcolor{white}{$y_1$}};
\node at (0,0)[yshift=1.5cm,xshift=-1.25cm,rectangle,,draw=blue!50,fill=blue!80,minimum size=8mm] (c2) {\textcolor{white}{$y_2$}};
\node at (0,0)[yshift=1.5cm,xshift=.25cm] (dot2) {$\dots$};
\node at (0,0)[yshift=1.5cm,xshift=1.5cm,rectangle,,draw=blue!50,fill=blue!80,minimum size=8mm] (cm) {\textcolor{white}{$y_m$}};

\node at (0,0)[yshift=-1.5cm,xshift=-3.5cm,circle,,draw=orange!50,fill=orange!80,minimum size=5mm] (p1) {$x_1$};
\node at (0,0)[yshift=-1.5cm,xshift=-2.25cm,circle,,draw=orange!50,fill=orange!80,minimum size=5mm] (p2) {$x_2$};
\node at (0,0)[yshift=-1.5cm,xshift=-1cm,circle,,draw=orange!50,fill=orange!80,minimum size=5mm] (p3) {$x_3$};
\node at (0,0)[yshift=-1.5cm,xshift=1cm] (dot) {$\dots$};
\node at (0,0)[yshift=-1.5cm,xshift=2.5cm,circle,,draw=orange!50,fill=orange!80,minimum size=5mm] (pn) {$x_n$};

\draw[line] (p1.north)--(c1.south);
\draw[line] (p1.north)--(c2.south);
\draw[line] (p2.north)--(c1.south);
\draw[line] (p2.north)--(cm.south);
\draw[line] (p3.north)--(c1.south);
\draw[line] (p3.north)--(cm.south);
\draw[line] (pn.north)--(c2.south);
\draw[line] (pn.north)--(cm.south);
\end{tikzpicture}
\caption{A bipartite graph that represents the constraints on the training set.\label{single_level_net}}

\end{figure}


\subsubsection{The recall phase}
In the recall phase, the neural network should retrieve the correct memorized pattern from a possibly corrupted version. In this case, the states of the pattern neurons $x_1,x_2,\dots,x_n$
are initialized with the given (noisy) input pattern. Here, we assume that the noise is integer valued and additive\footnote{It must be mentioned that neural states below $0$ and above
$Q-1$ will be clipped to $0$ and $Q-1$, respectively. This is biologically justified as the firing rate of neurons can not exceed an upper bound and of course can not be less than zero.}. Therefore, assuming the input to the network is a corrupted version of pattern $x^\mu$, the state of the pattern nodes are $x = x^\mu
+ z$, where $z$ is the noise. Now the neural network should use the given states together with the fact that $Wx^\mu = 0$ to retrieve pattern $x^\mu$, i.e. it should estimate $z$ from $Wx =
Wz$ and return $x^{\mu} = x-z$. Any algorithm designed for this purpose should be simple enough to be implemented by neurons. Therefore, our objective is to find a simple algorithm capable of
eliminating noise using only neural operations.
\subsection{Related Works}\label{sec:related}
Designing a neural associative memory has been an active area of research for the past three decades. Hopfield was the first to design an artificial neural associative memory in his seminal work in
1982 \cite{hopfield}. The so-called Hopfield network is inspired by Hebbian learning \cite{hebb} and is composed of binary-valued ($\pm 1$) neurons, which together are able to memorize a
certain number of patterns. In our terminology, the Hopfield network corresponds to a $(\{-1,1\},n,n)$ neural model. The pattern retrieval capacity of a Hopfield network of $n$ neurons was
derived later by Amit et al. \cite{amit} and shown to be $0.13n$, under vanishing bit error probability requirement. Later, McEliece et al. \cite{mceliece} proved that under the requirement
of vanishing pattern error probability, the capacity of Hopfield networks  is $n/(2\log(n))) = O(n/\log(n))$. 

In addition to neural networks with online learning capability, offline methods have also been used to design neural associative memories. For instance, in \cite{venkatesh} the authors
assume the complete set of pattern is given in advance and calculate the weight matrix using the pseudo-inverse rule \cite{hertz} offline. In return, this approach helps them improve the
capacity of a Hopfield network to $n/2$, under vanishing pattern error probability condition, while being able to correct \emph{one bit} of error in the recall phase. Although this is a
significant improvement to the $n/\log(n)$ scaling of the pattern retrieval capacity in \cite{mceliece}, it comes at the price of much higher computational complexity and the lack of
gradual learning ability. 

While the connectivity graph of a Hopfield network is a complete graph, Komlos and Paturi \cite{Komlos} extended the work of McEliece to sparse neural graphs. Their results are of
particular interest as physiological data is also in favor of sparsely interconnected neural networks. They have considered a network in which each neuron is connected to $d$ other neurons,
i.e., a $d$-regular network. Assuming that the network graph satisfies certain connectivity measures, they prove that it is possible to store a linear number of \emph{random} patterns (in
terms of $d$) with vanishing bit error probability or $C = O(d/\log n)$ random patterns with vanishing pattern error probability. Furthermore, they show that in spite of the capacity
reduction, the error correction capability remains the same as the network can still tolerate a number of errors which is linear in $n$.

It is also known that the capacity of neural associative memories could be enhanced if the patterns are of \emph{low-activity} nature, in the sense that at any time instant many of the
neurons are silent \cite{hertz}. However, even these schemes fail when required to correct a fair amount of erroneous bits as the information retrieval is not better compared to that of
normal networks.

Extension of associative memories to non-binary neural models has also been explored in the past. Hopfield addressed the case of continuous neurons and showed that similar to the binary
case, neurons with states between $-1$ and $1$ can memorize a set of random patterns, albeit with less capacity \cite{hopfield_non_binary}. Prados and Kak considered a digital version of non-binary neural networks in which neural states could assume integer (positive and negative) values \cite{prados_non_binary}. They show that the storage capacity of such networks are in general larger than their binary peers. However, the capacity would still be less than $n$ in the sense that the proposed neural network can not have more than $n$ patterns that are stable states of the network, let alone being able to retrieve the correct pattern from corrupted input queries.

In \cite{Jankowski} the authors investigated a multi-state complex-valued neural associative memory for which the estimated capacity is $C < 0.15 n$. Under the same model but using a different learning method, Muezzinoglu et al.
\cite{Muezzinoglu1} showed that the capacity can be increased to $C = n$. However the complexity of the weight computation mechanism is prohibitive. To overcome this drawback, a Modified
Gradient Descent learning Rule (MGDR) was devised in \cite{Lee}. In our terminology, all these models are $(\{e^{2\pi j s /k}|0 \leq s \leq k-1 \},n,n)$ neural associative memories. 

Given that even very complex offline learning methods can not improve the capacity of binary or multi-sate neural associative memories, a group of recent works has made considerable efforts
to exploit the inherent structure of the patterns in order to increase capacity and improve error correction capabilities. Such methods focus merely on memorizing those patterns that have
some sort of inherent redundancy. As a result, they differ from previous methods in which the network was deigned to be able to memorize any random set of patterns. Pioneering this
approach, Berrou and Gripon \cite{gripon_istc} achieved considerable improvements in the pattern retrieval capacity of Hopfield networks, by utilizing Walsh-Hadamard sequences.
Walsh-Hadamard sequences are a particular type of low correlation sequences and were initially used in CDMA communications to overcome the effect of noise. The only slight downside to the
proposed method is the use of a decoder based on the winner-take-all approach which requires a separate neural stage, increasing the complexity of the overall method. Using low correlation
sequences has also been considered in \cite{SKGS}, where the authors introduced two novel mechanisms of neural association that employ binary neurons to memorize patterns belonging to
another type of low correlation sequences, called Gold family \cite{gold}. The network itself is very similar to that of Hopfield, with a slightly modified weighting rule. Therefore,
similar to a Hopfield network, the complexity of the learning phase is small. However, the authors failed to increase the pattern retrieval capacity beyond $n$ and it was shown that the
pattern retrieval capacity of the proposed model is $C=n$, while being able to correct a fair number of erroneous input bits.

Later, Gripon and Berrou came up with a different approach based on neural cliques, which increased the pattern retrieval capacity to $O(n^2)$ \cite{gripon_sparse}. Their method is based on
dividing a neural network of size $n$ into $c$ clusters of size $n/c$ each. Then, the messages are chosen such that only one neuron in each cluster is active for a given message. Therefore,
one can think of messages as a random vector of length $c \log (n/c)$, where the $\log (n/c)$ part specifies the index of the active neuron in a given cluster. The authors also provide a
learning algorithm, similar to that of Hopfield, to learn the pair-wise correlations within the patterns. Using this technique and exploiting the fact that the resulting patterns are very
sparse, they could boost the capacity to $O(n^2)$ while maintaining the computational simplicity of Hopfield networks.

In contrast to the pairwise correlation of the Hopfield model, Peretto et al. \cite{peretto} deployed \emph{higher order} neural models: the models in which the state of the
neurons not only depends on the state of their neighbors, but also on the correlation among them. Under this model, they showed that the storage capacity of a higher-order Hopfield network
can be improved to $C=O(n^{p-2})$, where $p$ is the degree of correlation considered. The main drawback of this model is the huge computational complexity required in the learning
phase, as one has to keep track of $O(n^{p-2})$ neural links and their weights during the learning period. 

Recently, the present authors introduced a novel model inspired by modern coding techniques in which a neural bipartite graph is used to memorize the patterns that belong to a subspace
\cite{KSS}. The proposed model can be also thought of as a way to capture higher order correlations in given patterns while keeping the computational complexity to a minimal level (since
instead of $O(n^{p-2})$ weights one needs to only keep track of $O(n^2)$ of them). Under the assumptions that the bipartite graph is known, sparse, and expander, the proposed algorithm
increased the pattern retrieval capacity to $C=O(a^n)$, for some $a > 1$, closing the gap between the pattern retrieval capacities achieved in neural networks and that of coding techniques. For completeness, this approach is presented in the appendix (along with the detailed proofs). The main drawbacks in the proposed approach were the lack of a learning algorithm as well as the expansion assumption on the neural graph.

In this paper, we focus on extending the results described in \cite{KSS} in several directions: first, we will suggest an iterative learning algorithm, to find the neural connectivity
matrix from the patterns in the training set. Secondly, we provide an analysis of the proposed error correcting algorithm in the recall phase and investigate its performance as a function
of input noise and network model. Finally, we discuss some variants of the error correcting method which achieve better performance in practice.  

It is worth mentioning that an extension of this approach to a multi-level neural network is considered in \cite{SK_ISIT2012}. There, the novel structure enables better error correction.
However, the learning algorithm lacks the ability to learn the patterns one by one and requires the patterns to be presented all at the same time in the form of a big matrix. In \cite{KSS_ICML2013} we have further extended this approach to a modular single-layer architecture with online learning capabilities. The modular structure makes the recall algorithm much more efficient while the online learning enables the network to learn gradually from examples. The learning algorithm proposed in this paper is also virtually the same as the one we proposed in \cite{KSS_ICML2013}, giving it the advantage of 

Another important point to note is that learning linear constraints by a neural network is hardly a new topic as one can learn a matrix orthogonal to a set of patterns in the training set
(i.e., $Wx^\mu = 0$) using simple neural learning rules (we refer the interested readers to \cite{xu} and \cite{oja}). However, to the best of our knowledge, finding such a matrix subject
to the sparsity constraints has not been investigated before. This problem can also be regarded as an instance of compressed sensing \cite{candes}, in which the measurement matrix is given
by the big patterns matrix $\mc{X}_{C \times n}$ and the set of measurements are the constraints we look to satisfy, denoted by the tall vector $b$, which for simplicity reasons we assume
to be all zero. Thus, we are interested in finding a sparse vector $w$ such that $\mc{X}w = 0$. Nevertheless, many decoders proposed in this area are very complicated and cannot be
implemented by a neural network using simple neuron operations. Some exceptions are \cite{donoho_amp} and \cite{tropp} which are closely related to the learning algorithm proposed in this
paper. 


\subsection{Solution Overview}
Before going through the details of the algorithms, let us give an overview of the proposed solution. To learn the set of given patterns, we have adopted the neural learning algorithm
proposed in \cite{oja_proof} and modified it to favor sparse solutions. In each iteration of the algorithm, a random pattern from the data set is picked and the neural weights corresponding
to constraint neurons are adjusted is such a way that the projection of the pattern along the current weight vectors is reduced, while trying to make the weights sparse as well.  

In the recall phase, we exploit the fact that the learned neural graph is sparse and orthogonal to the set of patterns. Therefore, when a query is given, if it is not orthogonal to the connectivity
matrix of the weighted neural graph, it is noisy. We will use the sparsity of the neural graph to eliminate this noise using a simple iterative algorithm. In each iteration, there is a set
of violated constraint neurons, i.e. those that receive a non-zero sum over their input links. These nodes will send feedback to their corresponding neighbors among the pattern neurons,
where the feedback is the sign of the received input-sum. At this point, the pattern nodes that receive feedback from a majority of their neighbors update their state according to the sign
of the sum of received messages. This process continues until noise is eliminated completely or a failure is declared. 

In short, we propose a neural network with online learning capabilities which uses only neural operations to memorize an exponential number of patterns.

\section{Learning Phase}\label{section_learning}
Since the patterns are assumed to be coming from a subspace in the $n$-dimensional space, we adapt the algorithm proposed by Oja and Karhunen \cite{oja_proof} to learn the null-space basis
of the subspace defined by the patterns. In fact, a very similar algorithm is also used in \cite{xu} for the same purpose. However, since we need the basis vectors to be sparse (due to
requirements of the algorithm used in the recall phase), we add an additional term to penalize non-sparse solutions during the learning phase.

Another difference with the proposed method and that of \cite{xu} is that the learning algorithm proposed in \cite{xu} yields dual vectors that form an orthogonal set. Although one can
easily extend our suggested method to such a case as well, we find this requirement unnecessary in our case. This gives us the additional advantage to make the algorithm \emph{parallel} and
\emph{adaptive}. Parallel in the sense that we can design an algorithm to learn one constraint and repeat it several times in order to find all constraints with high probability. And
adaptive in the sense that we can determine the number of constraints on-the-go, i.e. start by learning just a few constraints. If needed (for instance due to bad performance in the recall
phase), the network can easily learn additional constraints. This increases the flexibility of the algorithm and provides a nice trade-off between the time spent on learning and the
performance in the recall phase. Both these points make an approach biologically realistic. 

It should be mentioned that the core of our learning algorithm here is virtually the same as the one we proposed in \cite{KSS_ICML2013}. 

\subsection{Overview of the proposed algorithm} 
The problem to find one sparse constraint \emph{vector} $w$ is given by equations (\ref{main_problem_modified_objective}), (\ref{main_problem_modified_constraint}), in which pattern $\mu$ is denoted by $x^\mu$.
%

\begin{subequations}\label{main_problem_modified}
\beq\label{main_problem_modified_objective}
\min \sum_{\mu =1 }^Cs \vert x^\mu \cdot w \vert^2 + \eta g(w).
\eeq
subject to:
\beq\label{main_problem_modified_constraint}
\Vert w \Vert_2 = 1
\eeq
\end{subequations}
In the above problem, $\cdot$ is the inner-product, $\Vert . \Vert_2$ represent the $\ell_2$ vector norm, $g(w)$ a penalty function to encourage sparsity and $\eta$ is a positive
constant. There are various ways to choose $g(w)$. For instance one can pick $g(w)$ to be $\Vert . \Vert_1$, which leads to $\ell_1$-norm penalty and is widely used in compressed sensing
applications \cite{donoho_amp}, \cite{tropp}. Here, we will use a different penalty function, as explained later. 

To form the basis for the null space of the patterns, we need $m=n-k$ vectors, which we can obtain by solving the above problem several times, each time from a random initial
point\footnote{It must be mentioned that in order to have exactly $m = n-k$ linearly independent vectors, we should pay some additional attention when repeating the proposed method several
time. This issue is addressed later in the paper.}.

As for the sparsity penalty term $g(w)$ in this problem, in this paper we consider the function
\begin{displaymath}
g(w) = \sum_{i = 1}^n \tanh(\sigma w_i^2),
\end{displaymath}
where $\sigma$ is chosen appropriately. Intuitively, $\tanh(\sigma w_i^2)$ approximates $|\hbox{sign}(w_i)|$ in $\ell_0$-norm. Therefore, the larger $\sigma$ is, the closer $g(w)$ will be
to $\Vert. \Vert_0$. By calculating the derivative of the objective function, and by considering the update due to each randomly picked pattern $x$, we will get the following iterative
algorithm:
\begin{subequations}\label{learning_sparse_null_one}
\beq\label{SGA_cost_modified}
y(t) = x(t) \cdot w(t)
\eeq
\beq\label{SGA_learning_rule_modified}
\tilde{w}(t+1) = w(t) - \alpha_t \left(2 y(t) x(t)  + \eta \Gamma(w(t))\right)
\eeq
\beq\label{w(t)}
w(t+1) = \frac{\tilde{w}(t+1)}{\Vert \tilde{w}(t+1) \Vert_2}
\eeq
\end{subequations}
In the above equations, $t$ is the iteration number, $x(t)$ is the sample pattern chosen at iteration $t$ uniformly at random from the patterns in the training set $\mc{X}$, and $\alpha_t$
is a small positive constant. Finally, $\Gamma(w): \mc{R}^n \rightarrow \mc{R}^n = \nabla g(w)$ is the gradient of the penalty term for non-sparse solutions. This function has the
interesting property that for very small values of $w_i(t)$, $\Gamma(w_i(t)) \simeq 2 \sigma w_i(t)$. To see why, consider the $i^{th}$ entry of the function $\Gamma(w(t)))$
\bed
\Gamma_i(w(t)) = \partial g(w(t)) /\partial w_i(t) = 2\sigma_t w_i(t) (1-\tanh^2(\sigma w_i(t)^2))
\eed
It is easy to see that $\Gamma_i(w(t)) \simeq 2 \sigma w_i(t)$ for relatively small $w_i(t)$'s. And for larger values of $w_i(t)$, we get $\Gamma_i(w(t)) \simeq 0$ (see Figure
\ref{sparsity_penalty}). Therefore, by proper choice of $\eta$ and $\sigma$, equation (\ref{SGA_learning_rule_modified}) suppresses small entries of $w(t)$ by pushing them towards zero,
thus, favoring sparser results. To simplify the analysis, with some abuse of notation, we approximate the function $\Gamma(w^{(\ell)}(t))$ with the following function:
\beq\label{soft_threshold_Gamma}
\Gamma_i(w^{(\ell)}(t)) = \left\{ \begin{array}{ll}
        w_i^{(\ell)}(t) & \mbox{if $|w_i^{(\ell)}(t)| \leq \theta_t$};\\        
        0 & \mbox{otherwise},\end{array} \right.
\eeq
where $\theta_t$ is a small positive threshold. 

\begin{figure}[h]
\begin{center}
\includegraphics[width=.55\textwidth]{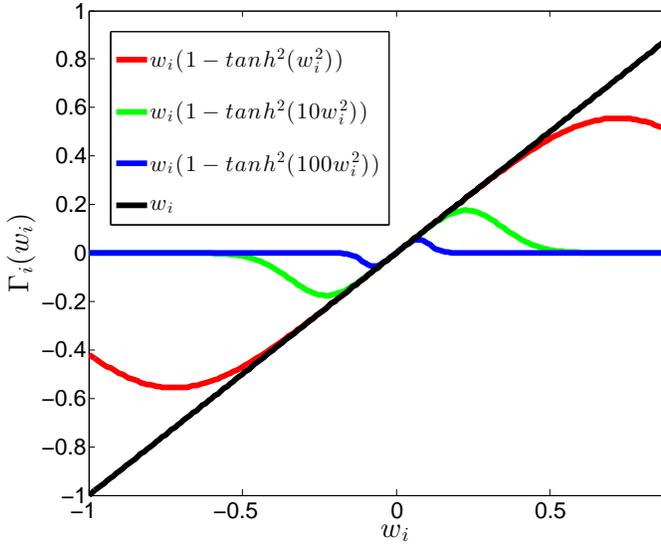}
\end{center}
\begin{center}
\caption{The sparsity penalty $\Gamma_i(w_i)$, which suppresses small values of the $i^{th}$ entry of $w$ in each iteration as a function of $w_i$ and $\sigma$. Note that the normalization
constant $2\sigma$ has been omitted here to make comparison with function $f = w_i$ possible. \label{sparsity_penalty}}
\end{center}
\end{figure}

Following the same approach as \cite{oja_proof} and assuming $\alpha_t$ to be small enough such that equation (\ref{w(t)}) can be expanded as powers of $\alpha_t$, we can approximate
equation (\ref{learning_sparse_null_one}) with the following simpler version:
\begin{subequations}\label{learning_sparse_null_main}
\beq\label{SGA_cost_main}
y(t) = x(t) \cdot w(t)
\eeq
\beq\label{learning_rule_final}
w(t+1) = w(t) - \alpha_t \left(y(t) \left(x(t) - \frac{y(t) w(t)}{\Vert w(t)\Vert_2^2}\right) + \eta \Gamma(w(t))  \right)
\eeq
\end{subequations}
In the above approximation, we also omitted the term $\alpha_t \eta \left(w(t)\cdot \Gamma(w(t))\right) w(t)$ since $w(t)\cdot \Gamma(w(t))$ would be negligible, specially as $\theta_t$ in equation (\ref{soft_threshold_Gamma}) becomes smaller. 

The overall learning algorithm for one constraint node is given by Algorithm \ref{algo_learning}. In words, in Algorithm \ref{algo_learning} $y(t)$ is the projection of $x(t)$ on the basis
vector $w(t)$. If for a given data vector $x(t)$, $y(t)$ is equal to zero, namely, the data is orthogonal to the current weight vector $w(t)$, then according to equation
(\ref{learning_rule_final}) the weight vector will not be updated. However, if the data vector $x(t)$ has some projection over $w(t)$ then the weight vector is updated towards the direction
to reduce this projection.

\begin{algorithm}[t]
\caption{Iterative Learning}
\label{algo_learning}
\begin{algorithmic}
\REQUIRE{ Set of patterns $x^\mu \in \mc{X}$ with $\mu=1,\dots,C$, stopping point $\varepsilon$.}
\ENSURE{$w$}
\WHILE{ $\sum_{\mu} \vert x^\mu \cdot w(t) \vert^2 >\varepsilon$ }
\STATE Choose $x(t)$ at random from patterns in $\mc{X}$
\STATE Compute $y(t) = x(t) \cdot w(t)$
\STATE Update $w(t+1) = w(t) - \alpha_t y(t) \left(x(t) - \frac{y(t) w(t)}{\Vert w(t)\Vert_2^2}\right) - \alpha_t \eta \Gamma(w(t))$.
\STATE $t \leftarrow t+1$.
\ENDWHILE
\end{algorithmic}
\end{algorithm}
Since we are interested in finding $m$ basis vectors, we have to do the above procedure \emph{at least} $m$ times in parallel.\footnote{In practice, we may have to repeat this process more
than $m$ times to ensure the existence of a set of $m$ linearly independent vectors. However, our experimental results suggest that most of the time, repeating $m$ times would be
sufficient.}

\begin{rem}
Although we are interested in finding a sparse graph, note that too much sparseness is not desired. This is because we are going to use the feedback sent by the constraint nodes to eliminate input
noise at pattern nodes during the recall phase. Now if the graph is too sparse, the number of feedback messages received by each pattern node is too small to be relied upon. Therefore, we must adjust the penalty
coefficient $\eta$ such that resulting neural graph is \emph{sufficiently} sparse. In the section on experimental results, we compare the error correction performance for different choices of
$\eta$.
\end{rem}

\subsection{Convergence analysis}
In order to prove that Algorithm \ref{algo_learning} converges to the proper solution, we use results from statistical learning. More specifically, we benefit from the convergence of
Stochastic Gradient Descent (SGD) algorithms \cite{bottu}. To prove the convergence, let $E(w) = \sum_{\mu} \vert x^\mu \cdot w \vert^2$ be the cost function we would like to minimize.
Furthermore, let $A = \mb{E}\{xx^T|x\in \mc{X}\}$ be the corelation matrix for the patterns in the training set. Therefore, due to uniformity assumption for the patterns in the training
set, one can rewrite $E(w) =  w^T A w$. Finally, denote $A_\mu = x^\mu(x^\mu)^T$. Now consider the following assumptions:
\ben
\item[A1.] $\Vert A \Vert_2 \leq \Upsilon < \infty$ and $\sup_{\mu} \Vert A_\mu \Vert_2 = \Vert x^\mu \Vert^2 \leq \zeta < \infty$.
\item[A2.] $\alpha_t > 0$, $\sum \alpha_t \rightarrow \infty$ and $\sum \alpha_t^2 < \infty$, where $\alpha_t$ is the small learning rate defined in \ref{learning_sparse_null_one}.
\een

The following lemma proves the convergence of Algorithm \ref{algo_learning} to a local minimum $w^*$.
\begin{lemma}\label{lemma_convergence}
Let assumptions A1 and A2 hold. Then, Algorithm \ref{algo_learning} converges to a local minimum $w^*$ for which $\nabla E(w^*) = 0$. 
\end{lemma}
\begin{proof}
To prove the lemma, we use the convergence results in \cite{bottu} and show that the required assumptions to ensure convergence holds for the proposed algorithm. For simplicity, these
assumptions are listed here:
\ben
\item The cost function $E(w)$ is three-times differentiable with continuous derivatives. It is also bounded from below.
\item The usual conditions on the learning rates are fulfilled, i.e. $\sum \alpha_t = \infty$ and $\sum \alpha_t^2 < \infty$.
\item The second moment of the update term should not grow more than linearly with size of the weight vector. In other words,
\bed
E(w) \leq a + b \Vert w \Vert_2^2
\eed
for some constants $a$ and $b$.
\item When the norm of the weight vector $w$ is larger than a certain horizon $D$, the opposite of the gradient $-\nabla E(W)$ points towards the origin. Or in other words:
\bed
\inf{\Vert w \Vert_2 > D} w \cdot \nabla E(w) >0
\eed
\item When the norm of the weight vector is smaller than a second horizon $F$, with $F>D$, then the norm of the update term $\left(2 y(t) x(t)  + \eta \Gamma(w(t))\right)$ is bounded
regardless of $x(t)$. This is usually a mild requirement:
\bed
\forall x(t) \in \mc{X},\ \sup_{\Vert w \Vert_2 \leq F} \Vert \left(2 y(t) x(t)  + \eta \Gamma(w(t))\right) \Vert_2 \leq K_0
\eed
\een

To start, assumption $1$ holds trivially as the cost function is three-times differentiable, with continuous derivatives. Furthermore, $E(w) \geq 0$. Assumption $2$ holds because of our
choice of the step size $\alpha_t$, as mentioned in the lemma description.

Assumption $3$ ensures that the vector $w$ could not escape by becoming larger and larger. Due to the constraint $\Vert w \Vert_2 = 1$, this assumption holds as well. 

Assumption $4$ holds as well because:
\beqa\label{assumption_iv}
\mb{E}_\mu \left(2 A_\mu w + \eta \Gamma(w)\right)^2  &=& 4w^T \mb{E}_\mu(A_\mu^2) w + \eta^2 \Vert \Gamma(w) \Vert_2^2 \nonumber \\
&+& 4\eta w^T \mb{E}_\mu(A_\mu) \Gamma(w) \nonumber \\
&\leq& 4 \Vert w \Vert_2^2 \zeta^2 + \eta^2 \Vert w \Vert_2^2 + 4 \eta \Upsilon \Vert w \Vert_2^2  \nonumber \\
&=&  \Vert w \Vert_2^2 (4\zeta^2 + 4 \eta \Upsilon +\eta^2) 
\eeqa

Finally, assumption $5$ holds because:
\beqa\label{assumption_v}
\Vert 2A_\mu w + \eta \Gamma(w) \Vert_2^2 &=&4 w^T A_\mu^2 w + \eta^2 \Vert \Gamma(w) \Vert_2^2 \nonumber \\
&+& 4\eta w^T A_\mu \Gamma(w) \nonumber \\
&\leq& \Vert w \Vert_2^2 (4\zeta^2 + 4 \eta \zeta +\eta^2) 
\eeqa
Therefore, $\exists F>D$ such that as long as $\Vert w \Vert_2^2 < F$:
\beq
\sup_{\Vert w \Vert_2^2 < E} \Vert 2A_\mu w + \eta \Gamma(w) \Vert_2^2 \leq (2\zeta + \eta)^2F = \hbox{constant}
\eeq

Since all necessary assumptions hold for the learning algorithm \ref{algo_learning}, it converges to a local minimum where $\nabla E(w^*) = 0$.
\end{proof}

Next, we prove the desired result, i.e. the fact that at the local minimum, the resulting weight vector is orthogonal to the patterns, i.e. $Aw = 0$. 
\begin{theorem}
In the local minimum where $\nabla E(w^*) = 0$, the optimal vector $w^*$ is orthogonal to the patterns in the training set.
\end{theorem}
\begin{proof}
Since $\nabla E(w^*) = 2Aw^* + \eta \Gamma(w^*) = 0$, we have:
\beq\label{lemma_2_proof}
w^* \cdot \nabla E(w^*) = 2(w^*)^TAw^* + \eta w^* \cdot \Gamma(w^*) 
\eeq
The first term is always greater than or equal to zero. Now as for the second term, we have that $|\Gamma(w_i)| \leq |w_i|$ and $\hbox{sign}(w_i) = \hbox{sign}(\Gamma(w_i))$, where $w_i$ is
the $i^{th}$ entry of $w$. Therefore, $ 0 \leq w^* \cdot \Gamma(w^*)  \leq \Vert w^* \Vert_2^2$. Therefore, both terms on the right hand side of (\ref{lemma_2_proof}) are greater than or
equal to zero. And since the left hand side is known to be equal to zero, we conclude that $(w^*)^T A w^* = 0$ and $\Gamma(w^*) = 0$. The former means $(w^*)^T A w^* = \sum_{\mu} (w^* \cdot
x^\mu)^2 = 0$. Therefore, we must have $w^* \cdot x^\mu = 0$, for all $\mu = 1, \dots,C$. This simply means that the vector $w^*$ is orthogonal to all the patterns in the training set.
\end{proof}

\begin{rem}
Note that the above theorem only proves that the obtained vector is orthogonal to the data set and says nothing about its degree of sparsity. The reason is that there is no guarantee that
the dual basis of a subspace be sparse. The introduction of the penalty function $g(w)$ in problem (\ref{main_problem_modified}) only encourages sparsity by suppressing the small entries of $w$,
i.e. shifting them towards zero if they are really small or leaving them intact if they are rather large. And from the fact that $\Gamma(w^*) =0$, we know this is true as the entries in
$w^*$ are either large or zero, i.e. there are no small entries. Our experimental results in section \ref{section_simulations} show that in fact this strategy works perfectly and the
learning algorithm results in sparse solutions. 
\end{rem}


\subsection{Avoiding the all-zero solution}
Although in problem (\ref{main_problem_modified}) we have the constraint $\Vert w \Vert_2 = 1$ to make sure that the algorithm does not converge to the trivial solution $w = 0$, due to
approximations we made when developing the optimization algorithm, we should make sure to choose the parameters such that the all-zero solution is still avoided. 

To this end, denote $w'(t) = w(t) - \alpha_t y(t)\left(x(t) - \frac{y(t) w(t)}{\Vert w(t) \Vert_2^2}\right)$ and consider the following inequalities:
\beqa
\Vert w(t+1)\Vert_2^2 &=& \Vert w(t) - \alpha_t y(t)\left(x(t) - \frac{y(t) w(t)}{\Vert w(t) \Vert_2^2} \right) \nonumber \\
&-& \alpha_t\eta \Gamma(w(t)) \Vert_2^2 \nonumber \\
&=& \Vert w'(t) \Vert^2+ \alpha_t^2 \eta^2 \Vert \Gamma(w(t)) \Vert^2 \nonumber \\
&-& 2\alpha_t \eta \Gamma(w(t))\cdot w'(t)   \nonumber \\
&\geq & \Vert w'(t) \Vert_2^2 - 2\alpha_t \eta \Gamma(w(t))\cdot w'(t)  \nonumber \\
\eeqa
Now in order to have $\Vert w(t+1)\Vert_2^2 > 0$, we must have that $ 2 \alpha_t \eta |\Gamma(w(t))^T w'(t)| < \Vert w'(t) \Vert_2^2$. Given that, $|\Gamma(w(t))\cdot w'(t)| \leq
\Vert w'(t) \Vert_2 \Vert \Gamma(w(t))\Vert_2$, it is therefore sufficient to have $2 \alpha_t \eta \Vert \Gamma(w(t)) \Vert_2 < \Vert w'(t) \Vert_2$. On the other hand, we have:
\beqa
\Vert w'(t) \Vert_2^2 &=& \Vert w(t) \Vert_2^2 + \alpha_t^2 y(t)^2 \Vert x(t) - \frac{y(t) w(t)}{\Vert w(t) \Vert_2^2} \Vert_2^2 \nonumber \\
&\geq& \Vert w(t) \Vert_2^2
\eeqa
As a result, in order to have $\Vert w(t+1) \Vert_2^2 > 0$, it is sufficient to have $2\alpha_t \eta \Vert \Gamma(w(t)) \Vert_2 < \Vert w(t) \Vert_2 $. Finally, since we have
$|\Gamma(w(t))| \leq |w(t)|$ (entry-wise), we know that $\Vert \Gamma(w(t)) \Vert_2 \leq \Vert w(t) \Vert_2$. Therefore, having $2 \alpha_t \eta < 1 \leq \Vert w(t)\Vert_2/\Vert
\Gamma(w(t)) \Vert_2$ ensures $\Vert w(t) \Vert_2 > 0 $.

\begin{rem}
Interestingly, the above choice for the function $w-\eta \Gamma(w)$ looks very similar to the soft thresholding function (\ref{soft_threshold_amp}) introduced in \cite{donoho_amp} to
perform iterative compressed sensing. The authors show that their choice of the sparsity function is very competitive in the sense that one can not get much better results by choosing other
thresholding functions. However, one main difference between their work and that of ours is that we enforce the sparsity as a penalty in equation (\ref{SGA_learning_rule_modified}) while
they apply the soft thresholding function in equation (\ref{soft_threshold_amp}) to the whole $w$, i.e. if the updated value of $w$ is larger than a threshold, it is left intact while it
will be put to zero otherwise.
\beq\label{soft_threshold_amp}
f_t(x) = \left\{ \begin{array}{ll}
        x- \theta_t & \mbox{if $x > \theta_t$};\\
        x+\theta_t & \mbox{if $x < -\theta_t$}\\
        0 & \mbox{otherwise}.\end{array} \right.
\eeq
where $\theta_t$ is the threshold at iteration $t$ and tends to zero as $t$ grows. 
\end{rem}

\subsection{Making the Algorithm Parallel}
In order to find $m$ constraints, we need to repeat Algorithm \ref{algo_learning} several times. Fortunately, we can repeat this process in parallel, which speeds up the algorithm and is
more meaningful from a biological point of view as each constraint neuron can act independently of other neighbors. Although doing the algorithm in parallel may result in linearly dependent
constraints once in a while, our experimental results show that starting from different random initial points, the algorithm converges to different distinct constraints most of the time.
And the chance of getting redundant constraints reduces if we start from a sparse random initial point. Besides, as long as we have enough distinct constraints, the recall algorithm in the
next section can start eliminating noise and there is no need to learn all the distinct basis vectors of the null space defined by the training patterns (albeit the performance improves as we learn more and more linearly independent constraints). Therefore, we will use the parallel version to
have a faster algorithm in the end.

\section{Recall Phase}\label{section_recall}
In the recall phase, we are going to design an iterative algorithm that corresponds to message passing on a graph. The algorithm exploits the fact that our learning algorithm resulted in the connectivity matrix of the neural graph which is sparse and orthogonal to the memorized patterns. Therefore, given a noisy version
of the learned patterns, we can use the feedback from the constraint neurons in Fig.~\ref{single_level_net} to eliminate noise. More specifically, the linear input sums to the constraint
neurons are given by the elements of the vector $W (x^\mu + z) = W x^\mu + W  z = W  z$, with $z$ being the integer-valued input noise (biologically speaking, the noise can be interpreted as a neuron
skipping some spikes or firing more spikes than it should). Based on observing the elements of $Wz$, each constraint neuron feeds back a message (containing info about $z$) to its neighboring pattern neurons. Based on this feedback, and exploiting the fact that $W$ is sparse, the pattern neurons update their states in order to reduce the noise $z$. 

It must also be mentioned that we initially assume \emph{assymetric} neural weights during the recall phase. More specifically, we assume the backward weight from constraint neuron $i$ to pattern
neuron $j$, denoted by $W^b_{ij}$ be equal to the sign of the weight from pattern neuron $i$ to constraint neuron $j$, i.e. $W^b_{ij} = \hbox{sign}(W_{ij})$, where $\hbox{sign(x)}$ is equal to $+1$, $0$ or $-1$ if $x>0$, $x=0$ or $x<0$, respectively. This assumption simplifies the
error correction analysis. Later in section \ref{section_practical}, we are going to consider another version of the algorithm which works with symmetric weights, i.e. $W^b_{ij} = W_{ij}$, and compare the performance
of all suggested algorithms together in section \ref{section_simulations}.

\subsection{The Recall Algorithms}
The proposed algorithm for the recall phase comprises a series of forward and backward iterations. Two different methods are suggested in this paper, which slightly differ from each other
in the way pattern neurons are updated. The first one is based on the Winner-Take-All approach (WTA) and is given by Algorithm \ref{algo_correction_expander_winner}. In this version, only
the pattern node that receives the highest amount of normalized feedback updates its state while the other pattern neurons maintain their current states. The normalization is done with
respect to the degree of each pattern neuron, i.e. the number of edges connected to each pattern neuron in the neural graph. The winner-take-all circuitry can be easily added to the neural
model shown in Figure \ref{single_level_net} using any of the classic WTA methods \cite{hertz}. 
\begin{algorithm}[t]
\caption{Recall Algorithm: Winner-Take-All}
\label{algo_correction_expander_winner}
\begin{algorithmic}[1]
\REQUIRE{Connectivity matrix $W$, iteration $t_{\max}$}
\ENSURE{$x_1,x_2,\dots,x_n$}
\FOR{$t = 1 \to t_{\max}$} 
\STATE \textit{Forward iteration:} Calculate the weighted input sum $ h_i = \sum_{j=1}^n W^b_{ij} x_j,$ for each constraint neuron $y_i$ and set:
\[ y_i = \left\{
\begin{array}{cc}
1, & h_i < 0\\
0, & h_i = 0\\
-1, & \hbox{otherwise}
\end{array} \right. .
\]
\STATE \textit{Backward iteration:} Each neuron $x_j$ with degree $d_j$ computes
\[ g^{(1)}_j = \frac{\sum_{i=1}^m W^b_{ij} y_i }{d_j}, g^{(2)}_j = \frac{\sum_{i=1}^m |W^b_{ij} y_i|}{d_j} \]
\STATE Find \[ j^* = \arg \max\limits_j g^{(2)}_j. \]
\STATE Update the state of winner $j^*$: set $ x_{j^*} = x_{j^*} + \hbox{sign} (g^{(1)}_{j^*}).$
 \STATE $t \gets t + 1$
\ENDFOR
\end{algorithmic}
\end{algorithm}

The second approach, given by Algorithm \ref{algo_correction_expander_majority_vote}, is much simpler: in every iteration, each pattern neuron decides locally whether or not to update its current state. More specifically, if the amount of feedback received by a pattern neuron exceeds a threshold, the neuron updates its state; otherwise, it remains unchanged.\footnote{Note that in order to
maintain the current value of a neuron in case no input feedback is received, we can add self-loops to pattern neurons in Figure \ref{single_level_net}. These self-loops are not shown in
the figure for clarity.} 
\begin{algorithm}[t]
\caption{Recall Algorithm: Majority-Voting}
\label{algo_correction_expander_majority_vote}
\begin{algorithmic}[1]

\REQUIRE{Connectivity matrix $W$, threshold $\varphi$, iteration $t_{\max}$}
\ENSURE{$x_1,x_2,\dots,x_n$}
\FOR{$t = 1 \to t_{\max}$} 
\STATE \textit{Forward iteration:} Calculate the weighted input sum $ h_i = \sum_{j=1}^n W^b_{ij} x_j,$ for each neuron $y_i$ and set:
\[ y_i = \left\{
\begin{array}{cc}
1, & h_i < 0\\
0, & h_i = 0\\
-1, & \hbox{otherwise}
\end{array} \right. .
\]
\STATE \textit{Backward iteration:} Each neuron $x_j$ with degree $d_j$ computes
\[ g^{(1)}_j = \frac{\sum_{i=1}^m W^b_{ij} y_i }{d_j}, g^{(2)}_j = \frac{\sum_{i=1}^m |W^b_{ij} y_i|}{d_j} \]
\STATE Update the state of each pattern neuron $j$ according to $x_j = x_j + \hbox{sign}(g^1_j)$
only if $|g^{(2)}_j| > \varphi$.
 \STATE $t \gets t + 1$
\ENDFOR
\end{algorithmic}
\end{algorithm}
In both algorithms, the quantity $g^{(2)}_j$ can be interpreted as the number of feedback messages received by pattern neuron $x_j$ from the constraint neurons. On the other hand, the sign of $g^{(1)}_j$
provides an indication of the sign of the noise that affects $x_j$, and $|g^{(1)}_j|$
indicates the confidence level in the decision regarding the sign of the noise.  

It is worthwhile mentioning that the Majority-Voting decoding algorithm is very similar to the Bit-Flipping algorithm of Sipser and Spielman to decode LDPC codes \cite{SipSpi} and a similar
approach  in \cite{hassibi} for compressive sensing methods. 

\begin{rem}
To give the reader some insight about why the neural graph should be sparse in order for the above algorithms to work, consider the backward iteration of both algorithms: it is based on
counting the fraction of received input feedback messages from the neighbors of a pattern neuron. In the extreme case, if the neural graph is complete, then a single noisy pattern neuron results in
the violation of all constraint neurons in the forward iteration. As a result, in the backward iteration all the pattern neurons receive feedback from their neighbors and it is impossible to
tell which of the pattern neuron is the noisy one. 

However, if the graph is sparse, a single noisy pattern neuron only makes some of the constraints unsatisfied. Consequently, in the recall phase only the nodes which share the neighborhood
of the noisy node receive input feedbacks. And the fraction of the received feedbacks would be much larger for the original noisy node. Therefore, by merely looking at the fraction of
received feedback from the constraint neurons, one can identify the noisy pattern neuron with high probability as long as the graph is sparse and the input noise is reasonable bounded.
\end{rem}

\subsection{Some Practical Modifications}\label{section_practical}
Although algorithm \ref{algo_correction_expander_majority_vote} is fairly simple and practical, each pattern neuron still needs two types of information: the number of received feedbacks and the
net input sum. Although one can think of simple neural architectures to obtain the necessary information, we can modify the recall algorithm to make it more practical and simpler. The trick
is to replace the degree of each node $x_j$ with the $\ell_1$-norm of the outgoing weights. In other words, instead of using $\Vert w_j \Vert_0 = d_j$, we use $\Vert w_j \Vert_1$.
Furthermore, we assume symmetric weights, i.e $W^b_{ij} = W_{ij}$. 

Interestingly, in some of our experimental results corresponding to \emph{denser graphs}, this approach performs much better, as will be illustrated in section \ref{section_simulations}. One possible reason behind this improvement might be the fact that using the $\ell_1$-norm instead of the $\ell_0$-norm in \ref{algo_correction_expander_majority_vote} will result in better differentiation
between two vectors that have the same number of non-zero elements, i.e. have equal $\ell_0$-norms, but differ from each other in the magnitude of the element, i.e. their $\ell_1$-norms
differ. Therefore, the network may use this additional information in order to identify the noisy nodes in each update of the recall algorithm.

\section{Performance Analysis}\label{section_error_analysis}
In order to obtain analytical estimates on the recall probability of error, we assume that the connectivity graph $W$ is sparse. With respect to this graph, we define the pattern and
constraint degree distributions as follows.
\begin{definition}
For the bipartite graph $W$, let $\lambda_i$ ($\rho_j$) denote the fraction of edges that are adjacent to pattern (constraint) nodes of degree $i$ ($j$). We call $\{\lambda_1,\dots,\lambda_m\}$ and
$\{\rho_1,\dots,\rho_n\}$ the pattern and constraint degree distribution form the edge perspective, respectively. Furthermore, it is convenient to define the degree distribution polynomials
as
\bed\label{node-degree}
\lambda(z) = \sum_{i} \lambda_i z^{i-1} \hbox{ and } \rho(z) = \sum_i \rho_i z^{i-1}.
\eed
\end{definition}

The degree distributions are determined after the learning phase is finished and in this section we assume they are given. Furthermore, we consider an ensemble of random neural graphs with
a given degree distribution and investigate the average performance of the recall algorithms over this ensemble. Here, the word "ensemble" refers to the fact that we assume having a number of \emph{random} neural graphs with
the given degree distributions and do the analysis for the average scenario. 

To simplify analysis, we assume that the noise entries are $\pm 1$. However, the proposed recall algorithms can work with any integer-valued noise and our experimental results suggest that this assumption is not necessary in practice. 

Finally, we assume that the errors do not cancel each other out in the constraint neurons (as long as the number of errors is fairly bounded). This is in fact a realistic assumption because
the neural graph is weighted, with weights belonging to the real field, and the noise values are integers. Thus, the probability that the weighted sum of some integers be equal to zero is
negligible. 

We do the analysis only for the Majority-Voting algorithms since if we choose the Majority-Voting update threshold $\varphi =1$, roughly speaking, we will have the winner-take-all
algorithm.\footnote{It must be mentioned that choosing $\varphi = 1$ does not yield the WTA algorithm exactly because in the original WTA, only one node is updated in each round. However,
in this version with $\varphi = 1$, all nodes that receive feedback from all their neighbors are updated. Nevertheless, the performance of the both algorithms is rather similar.}

As mentioned earlier, in this paper we will perform the analysis for general sparse bipartite graphs. However, restricting ourselves to a particular type of sparse graphs known as "expander" allows us to prove stronger results on the recall error probabilities. More details can be found in Appendix~\ref{appendix_expander_graph} and in \cite{KSS}.
However, since it is very difficult, if not impossible in certain cases, to make a graph expander during an iterative learning method, we focus on the more general case of sparse neural graphs.

To start the analysis, let $\mc{E}_t$ denote the set of erroneous pattern nodes at iteration $t$, and $\mc{N}(\mc{E}_t)$ be the set of constraint nodes that are connected to the nodes in
$\mc{E}_t$, i.e. these are the constraint nodes that have at least one neighbor in $\mc{E}_t$. In addition, let $\mc{N}^c(\mc{E}_t)$ denote the (complimentary) set of constraint neurons that do not have any
connection to any node in $\mc{E}_t$. Denote also the average neighborhood size of $\mc{E}_t$ by $S_t = \mb{E}(|\mc{N}(\mc{E}_t)|)$. Finally, let $\mc{C}_t$ be the set of correct pattern
nodes.

Based on the error correcting algorithm and the above notations, in a given iteration two types of error events are possible:
\ben
\item Type-1 error event: A node $x \in \mc{C}_t$ decides to update its value. The probability of this phenomenon is denoted by $P_{e_1}(t)$.
\item Type-2 error event: A node $x \in \mc{E}_t$ updates its value in the wrong direction. Let $P_{e_2}(t)$ denote the probability of error for this type.
\een
 
We start the analysis by finding explicit expressions and upper bounds on the average of $P_{e_1}(t)$ and $P_{e_2}(t)$ over all nodes as a function $S_t$. We then find an exact relationship for
$S_t$ as a function of $|\mc{E}_t|$, which will provide us with the required expressions on the average bit error probability as a function of the number of noisy input symbols, $|\mc{E}_0|$.
Having found the average bit error probability, we can easily bound the block error probability for the recall algorithm.

\subsection{Error probability - type 1}
To begin, let $P^x_1(t)$ be the probability that a node $x \in \mc{C}_t$ with degree $d_x$ updates its state. We have:
\beq
P^x_1(t) = \hbox{Pr}\{\frac{|\mc{N}(\mc{E}_t) \cap \mc{N}(x) |}{d_x} \geq \varphi \}
\eeq
where $\mc{N}(x)$ is the neighborhood of $x$. Assuming random construction of the graph and relatively large graph sizes, one can approximate $P^x_1(t)$ by
\beq\label{P_1_x}
P_1^x(t) \approx \sum_{i = \lceil \varphi d_x \rceil}^{d_x} {d_x \choose i} \left( \frac{S_t}{m}\right)^i \left( 1-\frac{S_t}{m}\right)^{d_x-i}.
\eeq 
In the above equation, $S_t/m$ represents the probabaility of having one of the $d_x$ edges connected to the $S_t$ constraint neurons that are neighbors of the erroneous pattern neurons.

As a result of the above equations, we have: 
\beq\label{P_e_1}
P_{e_1}(t) = \mb{E}_{d_x} (P^x_1(t)) ,
\eeq
where $\mb{E}_{d_x}$ denote the expectation over the degree distribution $\{\lambda_1,\dots,\lambda_m\}$. 

Note that if $\varphi =1$, the above equation simplifies to 
\bed
P_{e_1}(t) = \lambda\left(\frac{S_t}{m}\right)
\eed

\subsection{Error probability - type 2}
A node $x \in \mc{E}_t$ makes a wrong decision if the net input sum it receives has a different sign than the sign of noise it experiences. Instead of finding an exact relation, we bound
this probability by the probability that the neuron $x$ shares at least half of its neighbors with other neurons, i.e. $P_{e_2}(t) \leq \hbox{Pr}\{\frac{|\mc{N}(\mc{E^*}_t) \cap \mc{N}(x)
|}{d_x} \geq 1/2 \}$, where $\mc{E^*}_t = \mc{E}_t \setminus x$.
Letting $P^x_2(t) = \hbox{Pr}\{\frac{|\mc{N}(\mc{E^*}_t) \cap \mc{N}(x) |}{d_x} \geq 1/2 |\hbox{deg}(x) = d_x\}$, we will have:
\beq\label{P_2_x}
P^x_2(t) = \sum_{i = \lceil d_x/2 \rceil}^{d_x} {d_x \choose i} \left( \frac{S^*_t}{m}\right)^i \left( 1-\frac{S^*_t}{m}\right)^{d_x-i}
\eeq
where $S^*_t = \mb{E}(|\mc{N}(\mc{E}^*_t)|)$

Therefore, we will have:
\beq\label{P_e_2}
P_{e_2}(t) \leq \mb{E}_{d_x} (P^x_2(t)) 
\eeq

Combining equations (\ref{P_e_1}) and (\ref{P_e_2}), the bit error probability at iteration $t$ would be
\beqa\label{P_bit}
P_b(t+1) &=& \hbox{Pr$\{x \in \mc{C}_t\}$} P_{e_1}(t) + \hbox{Pr$\{x \in \mc{E}_t\}$} P_{e_2}(t) \nonumber \\
&=& \frac{n-|\mc{E}_t|}{n} P_{e_1}(t) + \frac{|\mc{E}_t|}{n} P_{e_2}(t) 
\eeqa

And finally, the average block error rate is given by the probability that at least one pattern node $x$ is in error. Therefore:
\beq\label{P_e}
P_e(t) = 1-(1-P_b(t))^n
\eeq
Equation (\ref{P_e}) gives the probability of making a mistake in iteration $t$. Therefore, we can bound the overall probability of error, $P_E$, by setting $P_E = \lim_{t \rightarrow \infty} P_e(t)$. To this end, we have to recursively update $P_b(t)$ in equation (\ref{P_bit}) and using $|\mc{E}_{t+1}| \approx n P_b(t+1)$. However, since we have assumed that the noise values are $\pm 1$, we can provide an upper bound on the total probability of error by considering
\beqa\label{final_P_E}
P_{E} &\leq& P_e(1)
\eeqa
In other words, we assume that the recall algorithms either correct the input error in the first iteration or an error is declared. Obviously, this bound is not tight as in practice and one might be able to correct errors in later iterations. In fact simulation results confirm this expectation. However, this approach provides a nice analytical upper bound since it only depends on the initial number of noisy nodes. As the initial number of noisy nodes grow, the above bound becomes tight. Thus, in summary we have:
\beq
P_E \leq 1-(1-\frac{n-|\mc{E}_0|}{n} \bar{P}_1^x - \frac{|\mc{E}_0|}{n} \bar{P}_2^x )^n
\eeq
where $\bar{P}_i^x = \mb{E}_{d_x}\{P_i^x\}$ and $|\mc{E}_0|$ is the number of noisy nodes in the input pattern initially. 

\begin{rem}
One might hope to further simplify the above inequalities by finding closed form approximation of equations (\ref{P_1_x}) and (\ref{P_2_x}). However, as one expects, this approach leads to very loose and trivial bounds in many cases. Therefore, in our experiments shown in section
\ref{section_simulations} we compare simulation results to the theoretical bound derived using equations (\ref{P_1_x}) and (\ref{P_2_x}).
\end{rem}

Now, what remains to do is to find an expression for $S_t$ and $S^*_t$ as a function of $|\mc{E}_t|$. The following lemma will provide us with the required relationship.
\begin{lemma}
The average neighborhood size $S_t$ in iteration $t$ is given by:
\beq\label{neighborhood_size}
S_t = m\left(1-(1-\frac{\bar{d}}{m})^{|\mc{E}_t|}\right)
\eeq
where $\bar{d}$ is the average degree for pattern nodes.
\end{lemma}
\begin{proof}
The proof is given in Appendix~\ref{appendix_neighbor_size}.
\end{proof}

\section{Pattern Retrieval Capacity}\label{section_capacity}
It is interesting to see that, except for its obvious influence on the learning time, the number of patterns $C$ does not have any effect in the learning or recall algorithm. As long as
the patterns come from a subspace, the learning algorithm will yield a matrix which is orthogonal to all of the patterns in the training set. And in the recall phase, all we deal with is
$Wz$, with $z$ being the noise which is independent of the patterns. 

Therefore, in order to show that the pattern retrieval capacity is exponential with $n$, all we need to show is that there exists a "valid" training set $\mc{X}$ with $C$ patterns of length
$n$ for which $C \propto a^{rn}$, for some $a > 1$ and $0<r$. By valid we mean that the patterns should come from a subspace with dimension $k < n$ and the entries in the patterns should be
non-negative integers. The next theorem proves the desired result.

\begin{theorem}\label{theorem_exponential_solution}
Let $\mc{X}$ be a $C \times n$ matrix, formed by $C$ vectors of length $n$ with non-negative integers entries between $0$ and $Q-1$. Furthermore, let $k = rn$ for some $0<r<1$. Then, there
exists a set of such vectors for which $C = a^{rn}$, with $a > 1$, and $\hbox{rank}(\mc{X}) = k <n$.
\end{theorem}

\begin{proof}
The proof is based on construction: we construct a data set $\mc{X}$ with the required properties. To start, consider a matrix $G \in \mb{R}^{k \times n}$ with rank $k$ and $k = rn$, with
$0 < r < 1$. Let the entries of $G$ be non-negative integers, between $0$ and $\gamma-1$, with $\gamma \geq 2$. 

We start constructing the patterns in the data set as follows: consider a set of random vectors $u^\mu \in \mb{R}^k$, $\mu =1,\dots,C$, with integer-valued entries between $0$ and $\upsilon-1$, where $\upsilon \geq
2$.  We set the pattern $x^\mu \in \mc{X}$ to be $x^\mu = u^\mu \cdot G$, \emph{if} all the entries of $x^\mu$ are between $0$ and $Q-1$. Obviously, since both $u^\mu$ and $G$ have only
non-negative entries, all entries in $x^\mu$ are non-negative. Therefore, it is the $Q-1$ upper bound that we have to worry about. 

The $j^{th}$ entry in $x^\mu$ is equal to $x_j^\mu = u^\mu \cdot G_j$, where $G_j$ is the $j^{th}$ column of $G$. Suppose $G_j$ has $d_j$ non-zero elements. Then, we have: 
\bed
x_j^\mu = u^\mu \cdot G_j \leq d_j (\gamma-1) (\upsilon-1)
\eed

Therefore, denoting $d^* = \max_{j} d_j$, we could choose $\gamma$, $\upsilon$ and $d^*$ such that
\beq\label{capaci_inequal}
Q-1 \geq d^* (\gamma-1) (\upsilon-1)
\eeq
to ensure all entries of $x^\mu$ are less than $Q$. 

As a result, since there are $\upsilon^k$ vectors $u$ with integer entries between $0$ and $\upsilon -1$, we will have $\upsilon^k = \upsilon^{rn}$ patterns forming $\mc{X}$. Which means $C
= \upsilon^{rn}$, which would be an exponential number in $n$ if $\upsilon \geq 2$. 
\end{proof}

As an example, if $G$ can be selected to be a sparse $200\times 400$ matrix with $0/1$ entries (i.e. $\gamma = 2$) and $d^* = 10$, and $u$ is also chosen to be a vector with $0/1$ elements
(i.e. $\upsilon = 2$), then it is sufficient to choose $Q\geq 11$ to have a pattern retrieval capacity of $C = 2^{rn}$. 

\begin{rem}
Note that inequality (\ref{capaci_inequal}) was obtained for the worst-case scenario and in fact is very loose. Therefore, even if it does not hold, we will still be able to memorize a
very large number of patterns since a big portion of the generated vectors $x^\mu$ will have entries less than $Q$. These vectors correspond to the message vectors $u^\mu$ that are "sparse"
as well, i.e. do not have all entries greater than zero. The number of such vectors is a polynomial in $n$, the degree of which depends on the number of non-zero entries in $u^\mu$.
\end{rem}
\section{Simulation Results}\label{section_simulations}

\subsection{Simulation Scenario}
We have simulated the proposed learning and recall algorithms for three different network sizes $n = 200,400,800$, with $k = n/2$ for all cases. For each case, we considered a few different
setups with different values for $\alpha$, $\eta$, and $\theta$ in the learning algorithm \ref{algo_learning}, and different $\varphi$ for the Majority-Voting recall algorithm
\ref{algo_correction_expander_majority_vote}. For brevity, we do not report all the results for various combinations but present only a selection of them to give insight on the performance of the proposed algorithms. 

In all cases, we generated $50$ random training sets using the approach explained in the proof of theorem \ref{theorem_exponential_solution}, i.e. we generated a generator matrix $G$ at
random with $0/1$ entries and $d^* = 10$. We also used $0/1$ generating message words $u$ and put $Q = 11$ to ensure the validity of the generated training set. 

However, since in this setup we will have $2^{k}$ patterns to memorize, doing a simulation over all of them would take a lot of time. Therefore, we have selected a random sample sub-set $\mc{X}$
each time with size $C = 10^5$ for each of the $50$ generated sets and used these subsets as the training set. 

For each setup, we performed the learning algorithm and then investigated the average sparsity of the learned constraints over the ensemble of $50$ instances. As explained earlier, all the
constraints for each network were learned in parallel, i.e. to obtain $m = n-k$ constraints, we executed Algorithm \ref{algo_learning} from random initial points $m$ time. 

As for the recall algorithms, the error correcting performance was assessed for each set-up, averaged over the ensemble of $50$ instances. The empirical results are compared to the
theoretical bounds derived in Section \ref{section_error_analysis} as well. 


\subsection{Learning Phase Results}
In the learning algorithm, we pick a pattern from the training set each time and adjust the weights according to Algorithm \ref{algo_learning}. Once we have gone over all the patterns, we repeat this operation several times to make sure that update for one pattern does not adversely affect the other learned patterns. Let $t$ be the iteration number of the learning algorithm, i.e. the number of times we have gone over the
training set so far. Then we set $\alpha_t \propto \alpha_0/t$ to ensure the conditions of Theorem~\ref{lemma_convergence} is satisfied. Interestingly, all of the constraints converged in
at most two learning iterations for all different setups. Therefore, the learning is very fast in this case. 

Figure \ref{sparisty_percentage} illustrates the percentage of pattern nodes with the specified sparsity measure defined as $\varrho = \kappa/n$, where $\kappa$ is the number of non-zero
elements. From the figure we notice two trends. The first is the effect of sparsity threshold, which as it is increased, the network becomes sparser. The second one is the effect of
network size, which as it grows, the connections become sparser. 

\begin{figure}[h]
\begin{center}
\includegraphics[width=.5\textwidth]{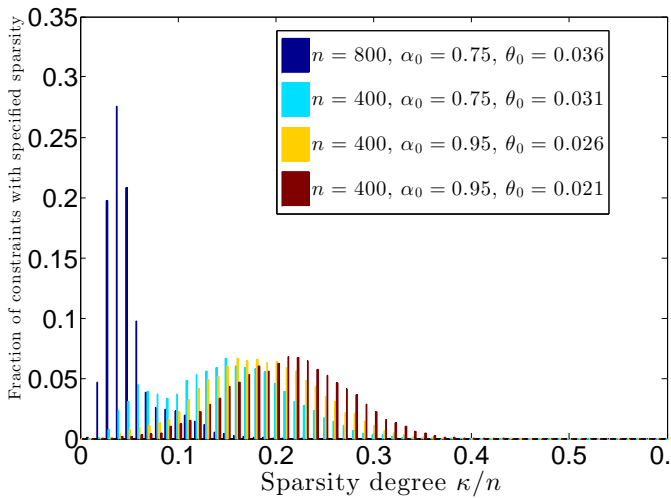}
\end{center}
\begin{center}
\caption{The percentage of variable nodes with the specified sparsity measure and different values of network sizes and sparsity thresholds. The sparsity measure is defined as $\varrho =
\kappa/n$, where $\kappa$ is the number of non-zero elements.\label{sparisty_percentage}}
\end{center}
\end{figure}


\subsection{Recall Phase Results}
For the recall phase, in each trial we pick a pattern randomly from the training set, corrupt a given number of its symbols with $\pm 1$ noise and use the suggested algorithm to correct the
errors. A pattern error is declared if the output does not match the correct pattern. We compare the performance of the two recall algorithms: Winner-Take-All (WTA) and Majority-Voting (MV). Table \ref{simul_param_recall} shows the simulation parameters in the recall phase for all scenarios (unless specified otherwise).  
\begin{table}[h]
\begin{center}
\label{simul_param_recall}
\caption{Simulation parameters}
\begin{tabular}{ |c| c| c|c|c| }
  \hline                       
  Parameter & $\varphi$ & $t_{\max}$ & $\varepsilon$ & $\eta$ \\ \hline
  Value & $1$ & $ 20 \Vert z \Vert_0 $ & $0.001$ & $1$\\
  \hline  
\end{tabular}
\end{center}
\end{table}

Figure \ref{PER_effect_of_theta_n_400_k_200} illustrates the effect of the sparsity threshold $\theta$ on the performance of the error correcting algorithm in the recall phase. Here, we have $n=400$ and $k = 200$. Two different sparsity thresholds are compared together, namely $\theta_t \propto 0.031/t$ and $\theta_t \propto 0.021/t$. Clearly, as network becomes sparser, i.e. $\theta$ increases, the performance of both recall algorithms improve. 
\begin{figure}[b]
\begin{center}
\includegraphics[width=.5\textwidth]{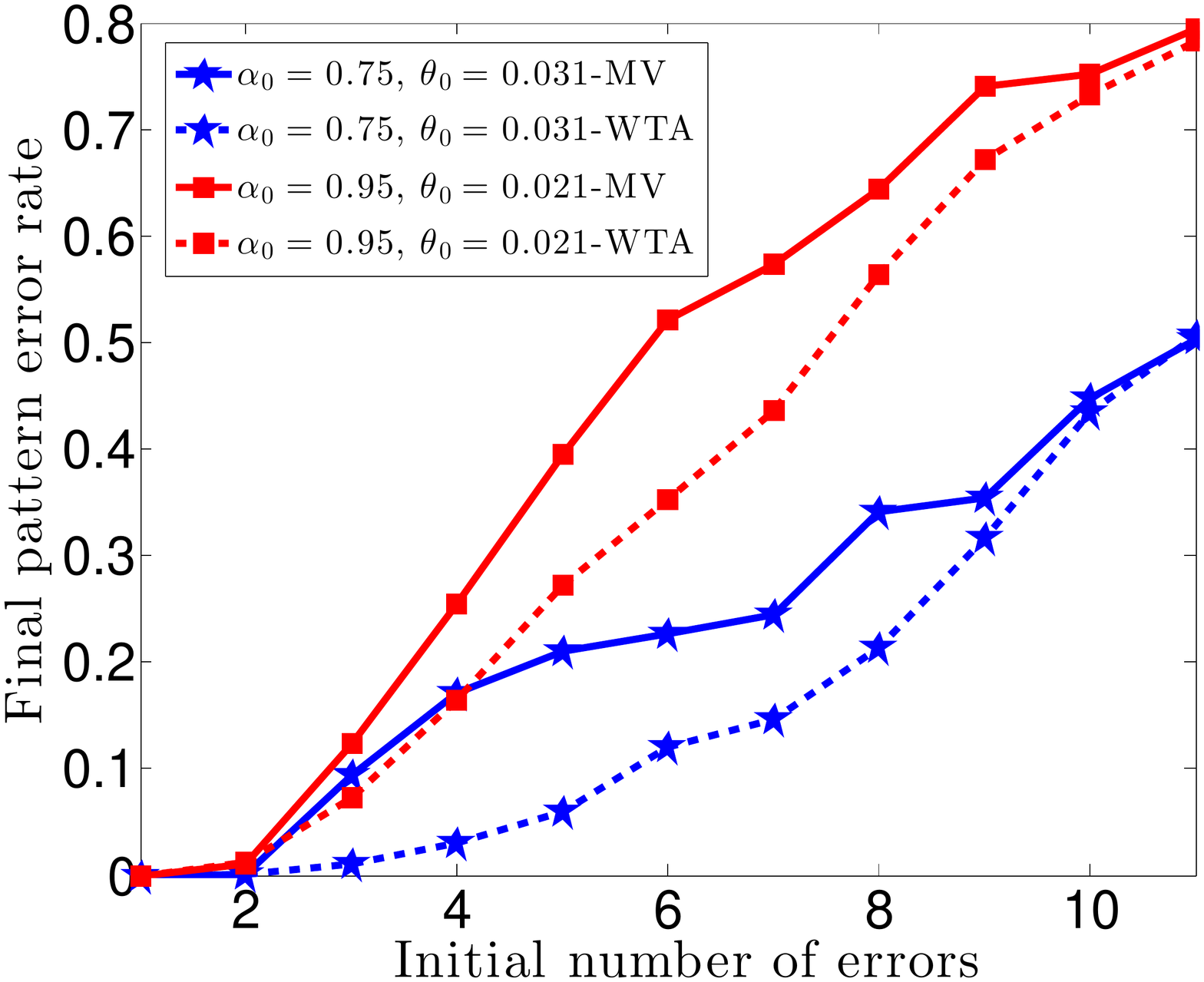}
\end{center}
\begin{center}
\caption{Pattern error rate against the initial number of erroneous nodes for two different values of $\theta_0$. Here, the network size is $n=400$ and $k =200$. The blue curves correspond to the sparser network (larger $\theta_0$) and clearly show a better performance. \label{PER_effect_of_theta_n_400_k_200}}
\end{center}
\end{figure}

In Figure \ref{PER_effect_of_size_800_400} we have investigated the effect of network size on the performance of recall algorithms by comparing the pattern error rates for two different network size, namely $n=800$ and $n=400$ with $k=n/2$ in both cases. As obvious from the figure, the performance improves to a great extent when we have a larger network. This is partially because of the fact that in larger networks, the connections are relatively sparser as well.
\begin{figure}[h]
\begin{center}
\includegraphics[width=.5\textwidth]{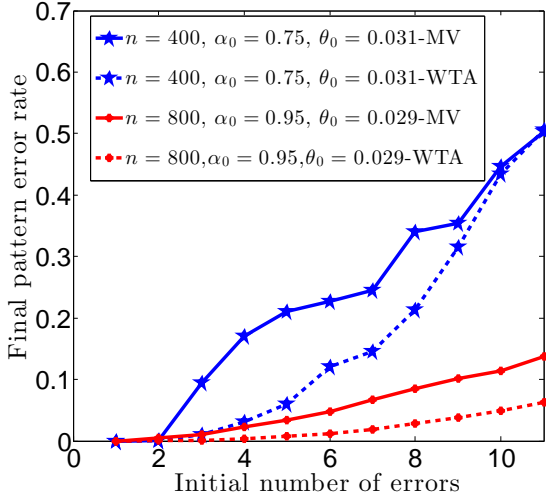}
\end{center}
\begin{center}
\caption{Pattern error rate against the initial number of erroneous nodes for two different network sizes $n=800$ and $k=400$. In both cases $k =n/2$. \label{PER_effect_of_size_800_400}}
\end{center}
\end{figure}

Figure~\ref{PER_theory} compares the results obtained in simulation with the upper bound derived in Section~\ref{section_error_analysis}. Note that as expected, the bound is quite loose since in
deriving inequality (\ref{P_e}) we only considered the first iteration of the algorithm.
\begin{figure}[b]
\begin{center}
\includegraphics[width=.5\textwidth]{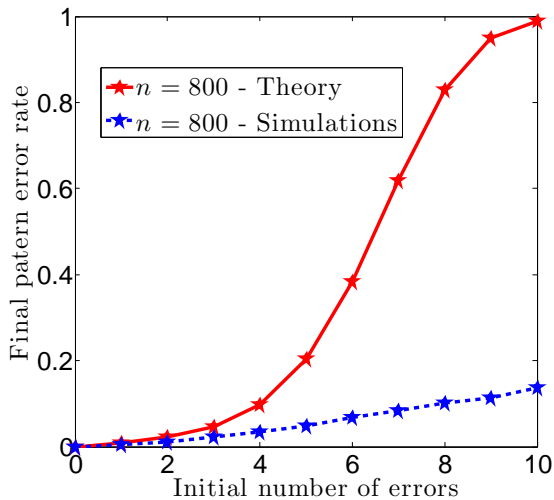}
\end{center}
\begin{center}
\caption{Pattern error rate against the initial number of erroneous nodes and comparison with theoretical upper bounds for $n = 800$, $k=400$, $\alpha_0 = 0.95$ and $\theta_0 =
0.029$.\label{PER_theory}}
\end{center}
\end{figure}

We have also investigated the tightness of the bound given in equation (\ref{final_P_E}) with simulation results. To this end, we compare $P_e(1)$ and $\lim_{t \rightarrow \infty} P_e(t)$ in our simulations for the case of $\pm 1$ noise. Figure~\ref{1st_itr_error_rate} illustrates the result and it is evident that allowing the recall algorithm to iterate improves the final probability of error to a great extent.
\begin{figure}[b]
\begin{center}
\includegraphics[width=.5\textwidth]{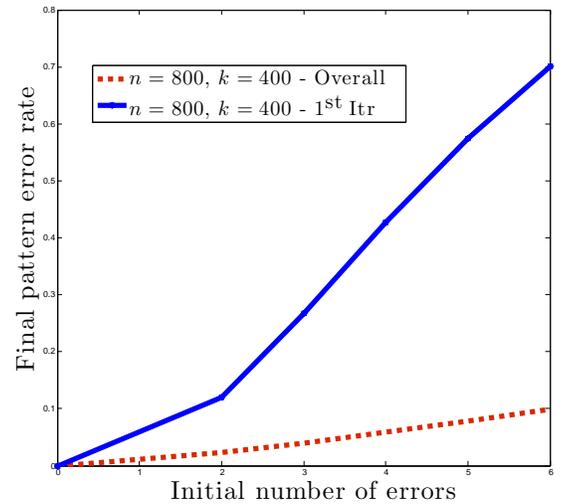}
\end{center}
\begin{center}
\caption{Pattern error rate in the first and last iterations against the initial number of erroneous nodes for $n = 800$, $k=400$, $\alpha_0 = 0.95$, $\theta_0 =
0.029$ and $\varphi = 0.99$.\label{1st_itr_error_rate}}
\end{center}
\end{figure}

Finally, we investigate the performance of the modified more practical version of the Majority-Voting algorithm, which was explained in Section~\ref{section_practical}. Figure~\ref{PER_effect_of_algorithm_200} compares the performance of the WTA and original MV algorithms with the modified version of MV algorithm for a network with size $n=200$, $k=100$ and learning parameters $\alpha_t \propto 0.45/t$, $\eta = 0.45$ and $\theta_t \propto 0.015/t$. The neural graph of this particular example is rather dense, because of small $n$ and sparsity threshold $\theta$. Therefore, here the modified version of the Majority-Voting algorithm performs better because of the extra information provided by the $\ell_1$-norm (than the $\ell_0$-norm in the original version of the Majority-Voting algorithm). However, note that we did not observe this trend for the other simulation scenarios where the neural graph was sparser.
\begin{figure}[h]
\begin{center}
\includegraphics[width=.5\textwidth]{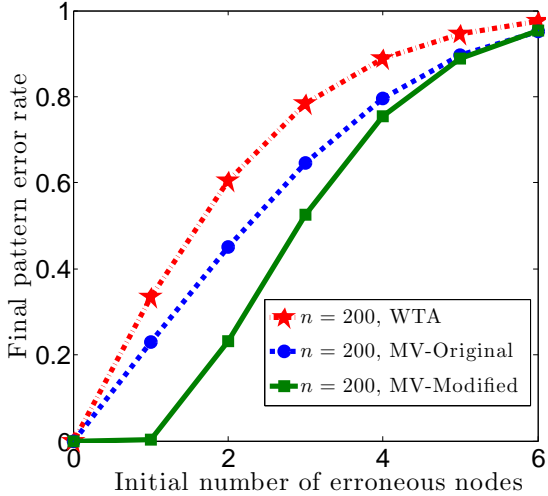}
\end{center}
\begin{center}
\caption{Pattern error rate against the initial number of erroneous nodes for two different values of $\theta_0$. Here, the network size is $n=400$ and $k =200$. The blue curves correspond to the sparser network (larger $\theta_0$) and clearly show a better performance. \label{PER_effect_of_algorithm_200}}
\end{center}
\end{figure}



\section{Conclusions and Future Works}\label{section_conclusion}
In this paper, we proposed a neural associative memory which is capable of exploiting inherent redundancy in input patterns to enjoy an exponentially large pattern retrieval capacity.
Furthermore, the proposed method uses simple iterative algorithms for both learning and recall phases which makes gradual learning possible and maintain rather good recall performances.
The convergence of the proposed learning algorithm was proved using techniques from stochastic approximation. We also analytically investigated the performance of the recall algorithm by
deriving an upper bound on the probability of recall error as a function of input noise. Our simulation results confirms the consistency of the theoretical results with those obtained in
practice, for different network sizes and learning/recall parameters.

Improving the error correction capabilities of the proposed network is definitely a subject of our future research. We have already started investigating this issue and proposed a different
network structure which reduces the error correction probability by a factor of $10$ in many cases \cite{SK_ISIT2012}. We are working on different structures to obtain even more robust
recall algorithms.

Extending this method to capture other sorts of redundancy, i.e. other than belonging to a subspace, will be another topic which we would like to explore in future.

Finally, considering some practical modifications to the learning and recall algorithms is of great interest. One good example is simultaneous learn and recall capability, i.e. to have a
network which learns a subset of the patterns in the subspace and move immediately to the recall phase. Now during the recall phase, if the network is given a noisy version of the patterns
previously memorized, it eliminates the noise using the algorithms described in this paper. However, if it is a new pattern, i.e. one that we have not learned yet, the network adjusts the
weights in order to learn this pattern as well. Such model is of practical interest and closer to real-world neuronal networks. Therefore, it would be interesting to design a network with
this capability while maintaining good error correcting capabilities and large pattern retrieval capacities.

\section*{Acknowledgment}
The authors would like to thank Prof. Wulfram Gerstner and his lab members, as well as Mr. Amin Karbasi for their helpful comments and discussions. This work was supported by Grant
228021-ECCSciEng of the European Research Council.

\appendices
\section{Average neighborhood size}\label{appendix_neighbor_size}
In this appendix, we find an expression for the average neighborhood size for erroneous nodes, $S_t = \mb{E}(|\mc{N}(\mc{E}_t)|)$. Towards this end, we assume the following procedure for
constructing a right-irregular bipartite graph:
\bit
\item In each iteration, we pick a variable node $x$ with a degree randomly determined according to the given degree distribution.
\item Based on the given degree $d_x$, we pick $d_x$ constraint nodes uniformly at random \emph{with replacement} and connect $x$ to the constraint node.
\item We repeat this process $n$ times, until all variable nodes are connected. 
\eit
Note that the assumption that we do the process with replacement is made to simplify the analysis. This assumption becomes more exact as $n$ grows. 

Having the above procedure in mind, we will find an expression for the average number of constraint nodes in each construction round. More specifically, we will find the average number of
constraint nodes connected to $i$ pattern nodes at round $i$ of construction. This relationship will in turn yields the average neighborhood size of $|\mc{E}_t|$ erroneous nodes in
iteration $t$ of error correction algorithm described in section \ref{section_recall}.

With some abuse of notations, let $S_e$ denote the number of constraint nodes connected to pattern nodes in round $e$ of construction procedure mentioned above. We write $S_e$ recursively
in terms of $e$ as follows:
\beqa
S_{e+1} &=& \mb{E}_{d_x} \{ \sum_{j = 0}^{d_x} {d_x \choose j} \left(\frac{S_e}{m} \right)^{d_x - j} \left(1-\frac{S_e}{m} \right)^{j} (S_e + j) \} \nonumber \\
&=& \mb{E}_{d_x} \{  S_e + d_x(1-S_e/m)  \} \nonumber\\
&=& S_e + \bar{d} (1-S_e/m)
\eeqa
Where $\bar{d} = \mb{E}_{d_x}\{d_x\}$ is the average degree of the pattern nodes. In words, the first line calculates the average growth of the neighborhood when a new variable node is
added to the graph. The proceeding equalities directly follows from relationship on binomial sums. Noting that $S_1 = \bar{d}$, one obtains:
\beqa\label{neighbor_size}
S_t = m\left(1-(1-\frac{\bar{d}}{m})^{|\mc{E}_t|}\right)
\eeqa

In order to verify the correctness of the above analysis, we have performed some simulations for different network sizes and degree distributions obtained from the graphs returned by the
learning algorithm. We generated $100$ random graphs and calculated the average neighborhood size in each iteration over these graphs. Furthermore, two different network sizes were
considered $n=100,200$ and $m=n/2$ in all cases, where $n$ and $m$ are the number of pattern and constraint nodes, respectively. The result for $n = 100,m=50$ is shown in Figure
\ref{neighbor_size_fig_100}, where the average neighborhood size in each iteration is illustrated and compared with theoretical estimations given by equation (\ref{neighbor_size}). Figure
\ref{neighbor_size_fig_200} shows similar results for $n = 200$, $m=100$.
\begin{figure}[h]
\begin{center}
\includegraphics[width=.50\textwidth]{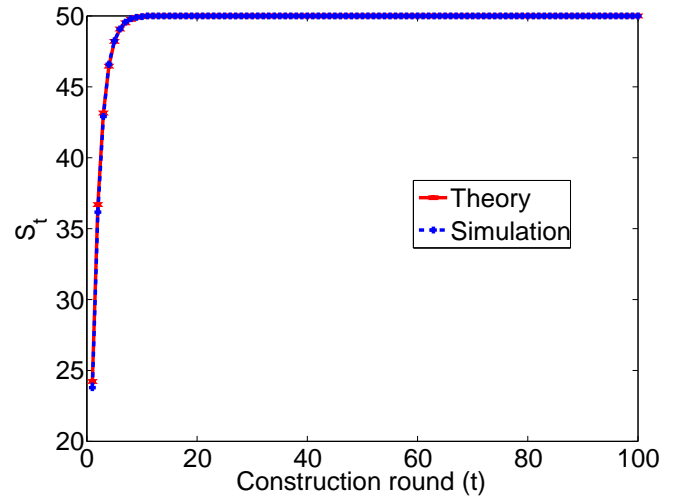}
\end{center}
\begin{center}
\caption{The theoretical estimation and simulation results for the neighborhood size of irregular graphs with a given degree-distribution for $n=100$, $m=50$ and over $2000$ random
graphs.\label{neighbor_size_fig_100}}
\end{center}
\end{figure}
In the figure,the dashed line shows the average neighborhood size over these graphs. The solid line corresponds to theoretical estimations. It is obvious that the theoretical value
is an exact approximation of the simulation results.

\begin{figure}[h]
\begin{center}
\includegraphics[width=.50\textwidth]{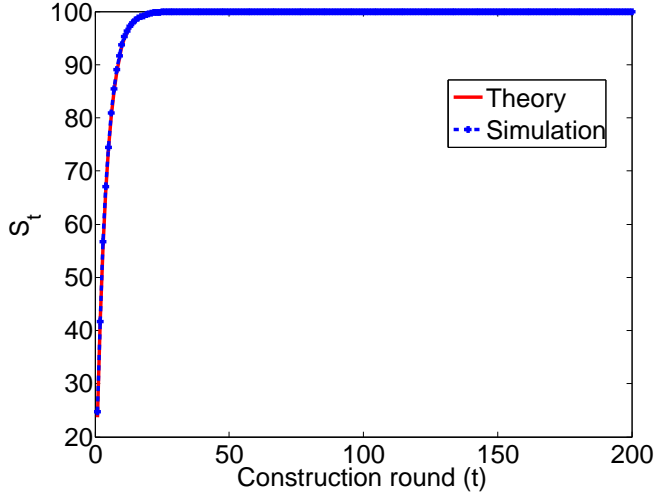}
\end{center}
\begin{center}
\caption{The theoretical estimation and simulation results for the neighborhood size of irregular graphs with a given degree-distribution for $n=200$, $m=100$ and over $2000$ random
graphs.\label{neighbor_size_fig_200}}
\end{center}
\end{figure}

\section{Expander Graphs}\label{appendix_expander}
This section contains the definitions and the necessary background on expander graphs. 
\begin{definition}\label{bipartite}
A regular $(d_p,d_c,n,m)$ bipartite graph $W$ is a
  bipartite graph between $n$ pattern nodes of degree $d_p$ and $m$ 
  constraint nodes of degree $d_c$.
\end{definition}

\begin{definition}
An $(\alpha n, \beta d_p)$-expander is a $(d_p,d_c,n,m)$ bipartite
  graph such that for any subset $\mc{P}$ of pattern nodes with
  $|\mc{P}|<\alpha n$ we have $|\mc{N}(\mc{P})| > \beta d_p|\mc{P}|$ where
  $\mc{N}(\mc{P})$ is the set of neighbors of $\mc{P}$ among the constraint nodes.
\end{definition}

The following result from \cite{SipSpi} shows the existence of families of expander graphs with parameter values that are relevant to us.
\begin{theorem} \cite{SipSpi}
Let $W$ be a randomly chosen $(d_p,d_c)-$regular bipartite graph between $n$ $d_p-$regular vertices and $m = (d_p/d_c)$ $d_c-$regular vertices. Then for all $0 < \alpha < 1$, with high
probability, all sets of $\alpha n$ $d_p-$regular vertices in $W$ have at least
\[ n \left( \frac{d_p}{d_c} ( 1-(1-\alpha)^{d_c} ) - \sqrt{\frac{2d_c\alpha h(\alpha)}{\log_2 e}} \right) \]
neighbors, where $h(\cdot)$ is the binary entropy function. 
\end{theorem}

The following result from \cite{BurMil} shows the existence of families of expander graphs with parameter values that are relevant to us.
\begin{theorem} \label{th:expander_existance}
Let $d_c$, $d_p$, $m$, $n$ be integers, and let $\beta < 1 -
  1/d_p$. There exists a small $\alpha > 0$ such that if $W$ is a
  $(d_p,d_c,n,m)$ bipartite graph chosen uniformly at random from the
  ensemble of such bipartite graphs, then $W$ is an $(\alpha n, \beta
  d_p)$-expander with probability $1-o(1)$, where $o(1)$ is a term
  going to zero as $n$ goes to infinity.
\end{theorem}

\section{Analysis of the Recall Algorithms for Expander Graphs}\label{appendix_expander_graph}
\subsection{Analysis of the Winner-Take-All Algorithm}

We prove the error correction capability of the winner-take-all algorithm in two steps: first we
show that in each iteration, only pattern neurons that are corrupted by
noise will be chosen by the winner-take-all strategy to update their state. Then, we prove that the update is in
the right direction, i.e. toward removing noise from the neurons. 
\begin{lem}\label{lemma_correct_nodes_winner_take_all}
If the constraint matrix $W$ is an $(\alpha n, \beta d_p)$ expander, with $\beta > 1/2$,
and the original number of erroneous neurons are less than or equal to $2$, then in each iteration of
the winner-take-all algorithm only the corrupted pattern nodes
update their value and the other nodes remain intact. For $\beta = 3/4$, the
algorithm will always pick the correct node if we have two or fewer
erroneous nodes.
\end{lem}

\begin{proof}
If we have only one node $x_i$ in error, it is obvious that the
corresponding node will always be the winner of the winner-take-all
algorithm unless there exists another node that has the same set of
neighbors as $x_i$. However, this is impossible as because of the
expansion properties, the neighborhood of these two nodes must have
at least $2\beta d_p$ members which for $\beta >1/2$ is strictly greater than $d_p$. As a result, no two nodes can have the same neighborhood
and the winner will always be the correct node.

In the case where there are two erroneous nodes, say $x_i$ and $x_j$,
let $\mc{E}$ be the set $\{x_i,  x_j\}$ and $\mc{N}(\mc{E})$ be the corresponding
neighborhood on the constraint nodes side. Furthermore, assume $x_i$ and $x_j$ share $d_{p'}$ of
their neighbors so that $|\mc{N}(\mc{E})| = 2d_p - d_{p'}$. Now because of the expansion properties:
\begin{displaymath}
|\mc{N}(\mc{E})| = 2d_p - d_{p'} > 2\beta d_p \Rightarrow d_{p'} < 2(1-\beta) d_p.
\end{displaymath}

Now we have to show that there are no nodes other than $x_i$ and
$x_j$ that can be the winner of the winner-take-all algorithm. To
this end, note that only those nodes that are connected to $N(\mc{E})$
will receive some feedback and can hope to be the winner of the process.
So let's consider such a node $x_\ell$ that is connected to
$d_{p_\ell}$ of the nodes in $N(\mc{E})$. Let $\mc{E}'$ be $\mc{E} \cup \{x_\ell \}$ and
$N(\mc{E}')$ be the corresponding neighborhood. Because of the expansion
properties we have $|N(\mc{E}')| = d_p - d_{p_\ell} + |N(\mc{E})|> 3\beta
d_p$. Thus:
\begin{eqnarray*}
d_{p_\ell} &<& d_p + |N(\mc{E})| -3\beta d_p = 3d_p(1-\beta) - d_{p'}.
\end{eqnarray*}
Now, note that the nodes $x_i$ and $x_j$ will receive some feedback
from $2d_p - d_{p'}$ edges because we assume there is no noise cancellation due to the fact that neural weights are real-valued and noise entries are integers. Since $2d_p - d_{p'} >
3d_p(1-\beta) - d_{p'}$ for $\beta >1/2$, we conclude that $d_p -
d_{p'} > d_{p_\ell}$ which proves that no node outside $\mc{E}$ can be
picked during the winner-take-all algorithm as long as $|\mc{E}| \leq 2$
for $\beta > 1/2$.
\end{proof}

In the next lemma, we show that the state of erroneous neurons is updated 
in the direction of reducing the noise.
\begin{lem}\label{lemma_direction_winner_take_all}
If the constraint matrix $W$ is an $(\alpha n, \beta d_p)$ expander, with $\beta > 3/4$,
and the original number of erroneous neurons is less than or equal $e_{\min} = 2$, then in each iteration of
the winner-take-all algorithm the winner is updated toward reducing
the noise.
\end{lem}
\begin{proof}
When there is only one erroneous node, it is obvious that all its neighbors
agree on the direction of update and the node reduces the amount of
noise by one unit.

If there are two nodes $x_i$ and $x_j$ in error, since the number of
their shared neighbors is less than $2(1-\beta)d_p$ (as we proved in the
last lemma), then more than half of their neighbors would be unique if $\beta \geq 3/4$. These unique neighbors agree on the
direction of update. Therefore, whoever the winner is will be
updated to reduce the amount of noise by one unit.
\end{proof}

The following theorem sums up the results of the previous lemmas to
show that the winner-take-all algorithm is guaranteed to perform error correction.
\begin{theorem}\label{theorem_correctness_winner_take_all}
If the constraint matrix $W$ is an $(\alpha n, \beta d_p)$ expander, with $\beta \geq 3/4$,
then the winner-take-all algorithm is guaranteed to correct at least
$e_{\min} = 2 $ positions in error, irrespective of the magnitudes of the errors.
\end{theorem}

\begin{proof}
The proof is immediate from Lemmas~\ref{lemma_correct_nodes_winner_take_all} and \ref{lemma_direction_winner_take_all}.
\end{proof}

\subsection{Analysis of the Majority Algorithm}

Roughly speaking, one would expect the Majority-Voting algorithm to be sub-optimal in comparison to the winner-take-all strategy, since the pattern neurons need to make independent decisions,
and are not allowed to cooperate amongst themselves. In this subsection, we show that despite this restriction, the Majority-Voting algorithm is capable of error correction; the sub-optimality
in comparison to the winner-take-all algorithm can be quantified in terms of a larger expansion factor $\beta$ being required for the graph.

\begin{theorem}\label{theorem_correctness_bit_flipping}
If the constraint matrix $W$ is an $(\alpha n, \beta d_p)$ expander with $\beta > \frac{4}{5}$,
then the Majority-Voting algorithm with $\varphi = \frac{3}{5}$ is guaranteed to correct at least
two positions in error, irrespective of the magnitudes of the errors.
\end{theorem}

\begin{proof}
As in the proof for the winner-take-all case, we will show our result in two steps: first, by showing that for a suitable choice of the Majority-Voting threshold $\varphi$, that only the
positions in error are updated in each iteration, and that this update is towards reducing the effect of the noise.

\paragraph{Case 1}
First consider the case that only one pattern node $x_i$ is in error. Let $x_j$ be any other pattern node, for some $j \neq i$. Let $x_i$ and $x_j$ have $d_{p'}$ neighbors in common. As
argued in the proof of Lemma~\ref{lemma_correct_nodes_winner_take_all}, we have that
\begin{equation} \label{eq:d_p'}
d_{p'} < 2 d_p (1-\beta). 
\end{equation}
Hence for $\beta = \frac{4}{5}$, $x_i$ receives non-zero feedback from at least $\frac{3}{5}d_p$ constraint nodes, while $x_j$ receives non-zero feedback from at most $\frac{2}{5}d_p$
constraint nodes. In this case, it is clear that setting $\varphi = \frac{3}{5}$ will guarantee that only the node in error will be updated, and that the direction of this update is towards
reducing the noise.

\paragraph{Case 2}
Now suppose that two distinct nodes $x_i$ and $x_j$ are in error. Let $\mc{E} = \{ x_i, x_j \}$, and let $x_i$ and $x_j$ share $d_{p'}$ common neighbors. If the noise corrupting these two
pattern nodes, denoted by $z_i$ and $z_j$, are such that $\hbox{sign}(z_i) = \hbox{sign}(z_j)$, then both $x_i$ and $x_j$ receive $-\hbox{sign}(z_i)$ along all $d_p$ edges that they are
connected to during the backward iteration. Now suppose that $\hbox{sign}(z_i) \neq \hbox{sign}(z_j)$. Then $x_i$ ($x_j$) receives correct feedback from at least the $d_p - d_{p'}$ edges in
$\mc{N}(\{x_i\}) \backslash \mc{E}$ (resp. $\mc{N}(\{x_j\}) \backslash \mc{E}$) during the backward iteration. Therefore, if $d_{p'} < d_p/2$, the direction of update would be also correct and the
feedback will reduce noise during the update. And from equation (\ref{eq:d_p'}) we know that for $\beta = 4/5$, $d_{p'} \leq 2d_p/5 < d_p/2$. Therefore, the two noisy nodes will be
updated towards the correct direction.

Let us now examine what happens to a node $x_\ell$ that is different from the two erroneous nodes $x_i, x_j$. Suppose that $x_\ell$ is connected to $d_{p_\ell}$ nodes in $\mc{N}(\mc{E})$. From
the proof of Lemma~\ref{lemma_correct_nodes_winner_take_all}, we know that
\begin{eqnarray*} 
d_{p_\ell} &<& 3 d_p(1-\beta) - d_{p'}\\
&\leq& 3 d_p(1-\beta).
\end{eqnarray*}
Hence $x_\ell$ receives at most $3 d_p(1-\beta)$ non-zero messages during the backward iteration.

For $\beta > \frac{4}{5}$, we have that $d_p - 2d_p (1-\beta) > 3d_p (1-\beta)$. Hence by setting $\beta = \frac{4}{5}$ and $\varphi = [d_p - 2d_p (1-\beta)]/d_p = \frac{3}{5}$, it is clear
from the above discussion that we have ensured the following in the case of two erroneous pattern nodes:
\begin{itemize}
\item The noisy pattern nodes are updated towards the direction of reducing noise.
\item No pattern node other than the erroneous pattern nodes is updated.
\end{itemize}
\end{proof}

\subsection{Minimum Distance of Patterns}
Next, we present a sufficient condition such that the minimum Hamming distance\footnote{Two (possibly non-binary) $n-$length vectors $x$ and $y$ are said to be at a Hamming distance $d$
from each other if they are coordinate-wise equal to each other on all but $d$ coordinates.} between these exponential number of patterns is not too small. In order to prove such a result,
we will exploit the expansion properties of the bipartite graph $W$; our sufficient condition will be in terms of a lower bound on the parameters of the expander graph.
\begin{theorem} \label{th:min_dist} 
Let $W$ be a $(d_p,d_c,n,m)-$regular bipartite graph, that is an $(\alpha n, \beta d_p)$ expander. Let $\mc{X}$ be the set of patterns corresponding to the expander weight matrix $W$. If
\[ \beta > \frac{1}{2} + \frac{1}{4 d_p}, \]
then the minimum distance between the patterns is at least $\lfloor \alpha n \rfloor + 1$.
\end{theorem}
\begin{proof}
Let $d$ be less than $\alpha n$, and $W_i$ denote the $i^{th}$
column of $W$. If two patterns are at Hamming distance $d$ from each
other, then there exist non-zero integers $c_1, c_2,
\dots, c_d$ such that  
\begin{equation} \label{eq:linear_comb}
c_1 W_{i_1} + c_2 W_{i_2} + \cdots + c_d W_{i_d} = 0, 
\end{equation}
where $i_1,\dots, i_d$ are distinct integers between $1$ and $n$. Let $\mc{P}$ denote any set of pattern nodes of the graph represented by $W$, with $|\mc{P}| = d$. As in \cite{hassibi}, we
divide $\mc{N}(\mc{P})$ into two disjoint sets: $\mc{N}_{unique}(\mc{P})$ is the set of nodes in $\mc{N} (\mc{P})$ that are connected to only one edge emanating from $\mc{P}$, and
$\mc{N}_{shared} (\mc{P})$ comprises the remaining nodes of $\mc{N} (\mc{P})$ that are connected to more than one edge emanating from $\mc{P}$.  If we show that $|\mc{N}_{unique} (\mc{P})| > 0$
for all $\mc{P}$ with $|\mc{P}| = d$, then (\ref{eq:linear_comb}) cannot hold, allowing us to conclude that no two patterns with distance $d$ exist. Using the arguments in \cite[Lemma
1]{hassibi}, we obtain that
\[ |\mc{N}_{unique} (\mc{P})| > 2 d_p |\mc{P}| \left( \beta - \frac{1}{2} \right). \]
Hence no two patterns with distance $d$ exist if
\[ 2 d_p d \left( \beta - \frac{1}{2} \right) > 1 \Leftrightarrow \beta > \frac{1}{2} + \frac{1}{2 d_p d}. \]
By choosing $\beta > \frac{1}{2} + \frac{1}{4 d_p}$, we can hence ensure that the minimum distance between patterns is at least $\lfloor \alpha n \rfloor + 1$.
\end{proof}
\subsection{Choice of Parameters}
In order to put together the results of the previous two subsections and obtain a neural associative scheme that stores an exponential number of patterns and is capable of error correction,
we need to carefully choose the various relevant parameters. We summarize some design principles below.

\begin{itemize}
\item From Theorems~\ref{th:expander_existance} and \ref{th:min_dist}, the choice of $\beta$ depends on $d_p$, according to
$ \frac{1}{2} + \frac{1}{4 d_p} < \beta < 1 - \frac{1}{d_p}$.

\item Choose $d_c, Q,\upsilon, \gamma$ so that Theorem~\ref{theorem_exponential_solution} yields an exponential number of patterns.

\item For a fixed $\alpha$, $n$ has to be chosen large enough so that
  an $(\alpha n, \beta d_p)$ expander exists according to
  Theorem~\ref{th:expander_existance}, with $\beta \geq 3/4$ and so that $\alpha n/2 \geq
  e_{\min} = 2$.

\end{itemize}

Once we choose a judicious set of parameters according to the above requirements, we
have a neural associative memory that is guaranteed to recall an exponential number of patterns even if the input is corrupted by errors in two coordinates. Our simulation results will
reveal that a greater number of errors can be corrected in practice.

\bibliographystyle{IEEEtran} 
\bibliography{Neural_ref}
\end{document}